\documentclass[11pt]{amsart}
\usepackage{graphicx} 
\usepackage{amsmath,amsfonts,amsthm,amssymb,amscd}
\usepackage{lipsum}
\usepackage{xcolor}
\usepackage{amsfonts}
\usepackage{graphicx}
\usepackage{epstopdf}
\usepackage{algorithm}
\usepackage{algorithmic}
\usepackage{upgreek}
\usepackage{overpic}
\usepackage{enumerate}
\usepackage{enumitem}
\usepackage{bm}
\usepackage[foot]{amsaddr}
\usepackage[hidelinks]{hyperref}
\usepackage[letterpaper,top=3.5cm,bottom=3.5cm,left=3.5cm,right=3.5cm,marginparwidth=1.75cm]{geometry}

\newcommand{\cN}{\mathcal{N}}
\newcommand{\cL}{\mathcal{L}}

\newcommand{\bE}{\mathbb{E}}

\newcommand{\bR}{\mathbb{R}}

\newcommand{\bN}{\mathbb{N}}
\newcommand{\bZ}{\mathbb{Z}}

\newcommand{\sfN}{\mathsf{N}}

\newtheorem{theorem}{Theorem}[section]
\newtheorem{lemma}[theorem]{Lemma}
\newtheorem{assumption}[theorem]{Assumption}
\newtheorem{proposition}[theorem]{Proposition}

\newtheorem{example}[theorem]{Example}

\theoremstyle{remark}
\newtheorem{remark}[theorem]{Remark}
\numberwithin{equation}{section}
\newenvironment{newremark}[1]{%
    \begin{remark}#1}{%
    \Endofdef\end{remark}%
}
\newcommand{\xqed}[1]{%
    \leavevmode\unskip\penalty9999 \hbox{}\nobreak\hfill
    \quad\hbox{\ensuremath{#1}}}
\newcommand{\Endofdef}{\xqed{\lozenge}}

\usepackage{tikz}
\usepackage{mathdots}
\usepackage{yhmath}
\usepackage{cancel}
\usepackage{color}
\usepackage{siunitx}
\usepackage{array}
\usepackage{multirow}
\usepackage{tabularx}
\usepackage{extarrows}
\usepackage{booktabs}
\usetikzlibrary{fadings}
\usetikzlibrary{patterns}
\usetikzlibrary{shadows.blur}
\usetikzlibrary{shapes}

\definecolor{darkred}{rgb}{.7,0,0}

\definecolor{darkgreen}{rgb}{.15,.55,0}

\definecolor{darkblue}{rgb}{0,0,0.7}

\newcommand{\edit}[1]{{\color{black} #1}}

\title[Convergence of Unadjusted Langevin in High Dimensions]{Convergence of Unadjusted Langevin in High Dimensions: Delocalization of Bias}

\author{Yifan~Chen\textsuperscript{1}}
\author{Xiaoou~Cheng}
\author{Jonathan Niles-Weed}
\author{Jonathan Weare}
\address{Courant Institute, New York University, NY, USA}
\email{yifanchen@math.ucla.edu, chengxo@nyu.edu, jnw@cims.nyu.edu, weare@nyu.edu}
\address{\textsuperscript{1}Now at Department of Mathematics, University of California, Los Angeles, CA, USA}
\begin{document}
\maketitle
\vspace{-2em}
\begin{abstract}
The unadjusted Langevin algorithm is commonly used to sample probability distributions in extremely high-dimensional settings. However, existing analyses of the algorithm for strongly log-concave distributions suggest that, as the dimension $d$ of the problem increases, the number of iterations required to ensure convergence within a desired error in the $W_2$ metric  scales in proportion to $d$ or $\sqrt{d}$. In this paper, we argue that, despite this poor scaling of the $W_2$ error for the full set of variables,  the behavior for a \emph{small number} of variables can be significantly better: a number of iterations proportional to $K$, up to logarithmic terms in $d$,
often suffices for the algorithm to converge to within a desired $W_2$ error for all $K$-marginals.
We refer to this effect as 
\textit{delocalization of bias}. 
We show that the delocalization effect does not hold universally and prove its validity for Gaussian distributions and strongly log-concave distributions with certain sparse interactions. Our analysis relies on a novel $W_{2,\ell^\infty}$ metric to measure convergence. A key technical challenge we address is the lack of a one-step contraction property in this metric. Finally, we use asymptotic arguments to explore potential generalizations of the delocalization effect beyond the Gaussian and sparse interactions setting.
\end{abstract}


\section{Introduction}
Overdamped Langevin dynamics
\begin{equation}
\label{eqn-langevin}
\mathrm{d}X_t = -\nabla V(X_t) \mathrm{d}t + \sqrt{2}\mathrm{d}B_t\, ,
\end{equation}
has been used extensively
to sample from high dimensional probability distributions in applications ranging from
molecular dynamics to Bayesian inverse problems and data assimilation \cite{pavliotis2014stochastic}.
Here $V: \mathbb{R}^d \to \mathbb{R}$ is a function in $\mathbb{R}^d$ referred to as the potential and $B_t$ is the $d$-dimensional Brownian motion. The target distribution $\pi$ is proportional to $\exp(-V)$, which is the stationary distribution of \eqref{eqn-langevin}.

The overdamped Langevin Monte Carlo algorithm, also known as the unadjusted Langevin algorithm, is obtained by applying the Euler–Maruyama scheme to \eqref{eqn-langevin}:
\begin{equation}
\label{eqn-OLMC}
    X_{(k+1)h} = X_{kh}-h\nabla V(X_{kh}) + \sqrt{2}(B_{(k+1)h} - B_{kh})\, ,
\end{equation}
    where $h$ is the step size. The distribution of $X_{kh}$ is denoted by $\rho_{kh}$.
    
\subsection{Motivation} The unadjusted Langevin algorithm \eqref{eqn-OLMC} is biased, meaning that even when $k \to \infty$, the distribution $\rho_{kh}$ will not converge to the exact target distribution $\pi$ for any finite step size $h$. For strongly log-concave distributions, $\rho_{kh}$ converges in the $W_2$ metric to the stationary distribution of \eqref{eqn-OLMC}, $\pi_h$, at a rate that is independent of dimension (see, e.g., \cite{durmusmoulines2019highdim}). Consequently, the dependence of the error $W_2(\rho_{kh}, \pi)$ on dimension, $d$, is completely determined by the bias $W_2(\pi_h,\pi)$.
This bias depends on the step size $h$, which, according to state-of-the-art analyses, needs to scale proportionally to $1/d$ or $1/\sqrt{d}$ to achieve a bounded bias for any $d$. 
Since the iteration complexity typically scales as $1/h$ (up to logarithmic factors), this small step size implies a computational cost of order $\sqrt{d}$ or $d$.

Unbiased samplers based on \eqref{eqn-langevin}, such as the Metropolis-adjusted Langevin algorithm \cite{rossky1978brownian,roberts1996exponential} and proximal samplers \cite{Lee2021proximal,Chen2022proximal}, are available,
but these algorithms also require a small step size when $d$ is large in order to ensure that the acceptance rate is non-negligible; see the review in Section \ref{sec-subsec-literature-review}.  
In short, existing theoretical analyses of both the unadjusted Langevin algorithm and unbiased variants require step sizes scaling as $d^{-c}$ for some $c > 0$. In the former case, small step sizes are required to ensure small bias; in the latter, small step sizes are required to maintain a reasonable acceptance rate.

In the case of the unadjusted Langevin algorithm, however, these theoretical predictions seem at odds with abundant empirical evidence that the scheme samples efficiently in extremely high dimensions.
As an example, molecular dynamics simulations using integrators closely related to \eqref{eqn-OLMC} with up to billions of atoms are not uncommon \cite{Gapsys2024bigmd}. In these simulations, the step size is typically set to several femtoseconds, irrespective of the system size \cite{alma9998721193507871}.

In this paper, we argue that the mismatch between theory and practice is due, in large part, to the metric used to measure convergence. Standard metrics such as $W_2$ measure accuracy of the entire distribution, but practitioners are often interested in averages of functions involving a relatively small set of variables. \edit{Indeed, both statistical and physical models often require the inclusion of many latent variables (e.g. solvent variables in molecular dynamics simulations) that are themselves of no direct interest. 
}

Based on these observations, we aim to investigate the convergence behavior of the algorithm under alternative metrics designed to characterize the accuracy of low-dimensional marginals. We identify a new ``delocalization'' phenomenon for the bias of the unadjusted Langevin algorithm to show that, even in high-dimensional settings, the step size may not need to be very small if the quantities of interest depend on low-dimensional marginals only. Consequently, the iteration complexity required to achieve a bounded error exhibits a benign dependence on dimension $d$.

\subsection{Literature review} 
\label{sec-subsec-literature-review}
We begin by reviewing existing works on the analysis of Langevin algorithms, with a particular focus on the dependence on dimension $d$. Analysis of \eqref{eqn-OLMC} dates back to the work of Roberts and Tweedie \cite{roberts1996exponential}, where asymptotic properties such as ergodicity of the discrete Markov chain are studied. Asymptotic bias of the discrete SDEs can be investigated using Taylor's expansion and the Poisson equation \cite{talay1990second,talay1990expansion,bally1996law,mattingly2010convergence}. Non-asymptotic analysis of the algorithm \eqref{eqn-OLMC} aims to characterize the step size and iteration complexity to achieve bounded error under certain metrics; again, the iteration complexity typically scales inversely with the step size (up to logarithmic factors of $d$). Here, we focus on the scaling of the step size  with dimension. A large body of work considers strongly log-concave distributions:
\begin{assumption}
 \label{assumption-V-log-concave}
     Let $\pi \propto \exp(-V)$ and $V \in C^2(\bR^d)$. Assume $V$ is $\alpha$-strongly convex and $\beta$-smooth such that $\alpha I \preceq \nabla^2 V(x) \preceq \beta I$ at any $x \in \bR^d$. Here $0 < \alpha \leq \beta < \infty$. 
 \end{assumption}
 \vspace{-0.05em}
Under Assumption \ref{assumption-V-log-concave}, existing non-asymptotic bounds for the unadjusted Langevin algorithm \eqref{eqn-OLMC} mainly focus on metrics such as the total variation distance, the $W_2$ distance, and the Kullback-Leibler (KL) divergence; see \cite{Dalalyan2017theoretical, Dalalyan2017further, durmus2017nonasymptotic, durmusmoulines2019highdim, cheng2018convergence, durmus2019analysis}. In these studies, to maintain a bounded error or bias as the dimension $d$ increases, the step size $h$ must scale inversely with $d$, i.e., $h \sim 1/d$.
With additional smoothness assumptions on the Hessian of $V$, the bound on the step size can be improved to $h \sim 1/\sqrt{d}$; see \cite{li2022sqrtd} and the mean squared analysis framework developed in \cite{li2019stochastic}. Other convergence results exist, such as those using the $\chi^2$ and R\'enyi divergences as metrics \cite{Vempala2019, Erdogdu2022renyi}, under weaker assumptions like the log-Sobolev inequality \cite{Vempala2019,chewi2022analysis} or other conditions on $\pi$ \cite{chatterji2020langevin, erdogdu2021convergence, mou2022improved, lehec2023langevin}, and alternative discretization schemes \cite{shen2019randomized,he2020ergodicity}. All these results lead to power law scaling of the step size with $d$ (up to logarithmic terms) to attain a desired error for any $d$.

Metropolis-adjusted Langevin algorithms (MALA) \cite{roberts1996exponential} and proximal samplers \cite{Lee2021proximal} are two common unbiased schemes for the Langevin dynamics \eqref{eqn-langevin}. For MALA, the asymptotic analysis in \cite{Roberts1998scaling}  shows that, when $h \sim 1/d^{1/3}$, the algorithm admits a non-degenerate diffusion limit as $d \to \infty$, provided $\pi$ is a smooth product measure and the algorithm starts at stationarity; see also further generalizations beyond product measures and stationarity \cite{breyer2004optimal,christensen2005scaling,Pillai2012scaling, kuntz2018non} and note that  out of stationarity, the results suggest the scaling $h \sim 1/\sqrt{d}$. Non-asymptotic bounds on the mixing time of MALA have also been investigated in the literature \cite{Dwi2018fast, Chen2020fastmixinghmc, Chewi2021optimal, wu2022minimax, chen2023simple}. These bounds require a specific choice of step size, scaling as $h \sim 1/\sqrt{d}$ for MALA initialized at a warm start and $h \sim 1/d$ for a feasible start, up to $\log d$ terms; see also related studies on the lower bounds of the mixing time \cite{Chewi2021optimal, lee2021lower, wu2022minimax}.
   For proximal samplers \cite{Lee2021proximal}, a key component is the implementation of the restricted Gaussian oracle (RGO), and the step size $h$ needs to be small to ensure efficient implementation. Existing analyses suggest choosing $h \sim 1/d$ if the RGO is implemented via rejection sampling \cite{Chen2022proximal}, and $h \sim 1/\sqrt{d}$ if it is implemented via approximate rejection sampling \cite{Fan2023improveddim}.  

    In summary, for both MALA and proximal samplers, choosing the step size inversely proportional to  $d$ or $\sqrt d$ is necessary for efficient implementation and to achieve the desired mixing time bounds. For the unadjusted Langevin algorithm, similar scaling is required to attain a bounded error on the bias under the aforementioned metrics. 
    
    Under this context, the goal of this paper is to show that the requirement on the step size in the unadjusted Langevin algorithm can be significantly improved if we use an alternative $W_{2,\ell^{\infty}}$ metric to measure convergence.
    The convergence bounds we prove in this metric indicate that, in some situations, the bias of the unajusted Langevin algorithm is \emph{delocalized}, in the sense that the bias in individual coordinates is nearly dimension-free. 
    This is precisely the behavior observed when $\pi$ is a product measure. Our results show that this phenomenon holds in wider generality.
     
     References \cite{bou2023convergence} and \cite{durmus2021asymptotic} studied  target distributions for which larger stepsizes are possible in high-dimensions if the observable of interest has Lipschitz constant that scales as $1/\sqrt{d}$, 
     such as the averaged observable $f(x)=\frac{1}{d}\sum_{i=1}^d\Phi(x^{(i)})$ with $\ell^2$-Lipschitz $\Phi$.
     In these cases, a constant step size is sufficient to achieve a bounded error regardless of dimensionality.
    \subsection{Main results}
    \label{sec-intro-main-result}
    In this section, we outline our main results, including a novel metric $W_{2,\ell^{\infty}}$ to measure convergence, two positive and one negative example that illustrate the delocalization effects of bias, and our main theorem regarding the convergence result for strongly log-concave distributions with sparse interactions.
    \subsubsection{A new metric to measure convergence} 
    \label{sec-a-new-metric-measure-convergence}
     We introduce the following $W_{p,\ell^{\infty}}$ metric:
    \begin{equation}
        W_{p,\ell^{\infty}}(\mu, \nu) = \left( \min_{\gamma \in \Pi(\mu,\nu)} \int |x-y|^p_\infty \gamma({\rm d}x,{\rm d}y)\right)^{1/p}\, ,
    \end{equation}
where $|\cdot|_{\infty}$ is the $\ell^{\infty}$ norm of a vector and $\Pi(\mu,\nu)$ represents the set of measures in the joint space $\bR^d \times \bR^d$ that have marginals $\mu, \nu$. \edit{We also use $|M|_{\infty}$ to denote the $\ell^\infty \to \ell^\infty$ operator norm of a matrix $M \in \bR^{d\times d}$ throughout this article:
\begin{equation*}
    |M|_{\infty} = \sup_{x \in \bR^d: x \neq 0} \frac{|M x|_\infty}{|x|_\infty}\,.
\end{equation*}
Similarly, $|\cdot|_2$ stands for the $\ell^2$ norm and, when applied to matrices, to the corresponding operator norm.}

We note that $|x-y|_{\infty}\geq |x^{(j)} - y^{(j)}|$ for any $1\leq j \leq d$, where we use the superscript $(j)$ to denote the $j$-th component of a vector. This observation implies that $W_{p,\ell^{\infty}}(\mu, \nu)$ serves as an upper bound for the $W_p$ distance between one-dimensional marginals of $\mu$ and $\nu$. Moreover, since $K|x-y|_{\infty}^p\geq \sum_{t=1}^K|x^{(j_t)} - y^{(j_t)}|^p$ for any $1\leq j_t \leq d$, we have that $ K^{1/p} \cdot W_{p,\ell^{\infty}}(\mu, \nu)$ serves as the upper bound for the $W_p$ distance between any $K$-dimensional marginals of $\mu$ and $\nu$. In summary, the metric $W_{p,\ell^{\infty}}$ is capable of measuring the accuracy of low-dimensional marginals. In this paper, we specifically focus on the case $p=2$, i.e., the $W_{2,\ell^{\infty}}$ metric.

Another important observation is that 
\[W_{2,\ell^\infty}(\rho_{kh}, \pi) \leq W_{2,\ell^\infty}(\rho_{kh}, \pi_h) + W_{2,\ell^\infty}(\pi_h, \pi) \leq W_2(\rho_{kh}, \pi_h) + W_{2,\ell^\infty}(\pi_h, \pi)\, ,\]
which, combined with the dimension-independent contraction result in the $W_2$ metric, implies that the $W_{2,\ell^\infty}$ bias  governs the dependence of the $W_{2,\ell^{\infty}}$ error on $d$.

\edit{We note that Langevin dynamics is rotation invariant and not sensitive to the choice of basis, as is the standard $W_2$ metric. However, a distinct feature of the introduced $W_{2,\ell^\infty}$ metric is its dependence on the coordinate system. In fact, as we will see later, the delocalization effect studied in this paper is coordinate dependent, which also justifies the use of the coordinate-dependent $W_{2,\ell^\infty}$ metric.}

\subsubsection{Positive and negative examples: delocalization effect} How should we expect the $W_{2,\ell^{\infty}}$ bias of \eqref{eqn-OLMC} behave? To motivate the discussion, we first consider the examples of product measures and Gaussian measures. We begin with the product measure case for which a bound on $W_{2,\ell^{\infty}}(\pi_h,\pi)$ can be obtained by a contraction argument similar to the $W_2$ analysis (e.g., \cite{Dalalyan2017further}). We include a sketch of the argument for later reference in the sketch of the proof of our main result, Theorem \ref{thm-OLD-bias-under-sparsity} in Section \ref{sec-Convergence-bound-sparse-model}.
\begin{example}
\label{intro-product-measure-example}
    Consider $\pi \propto \exp(-V)$ where $V(x) = \sum_{i=1}^d V_i(x^{(i)})$ satisfies $\alpha \leq \nabla^2 V_i \leq \beta$.  Then, for \edit{$h \leq 1/\beta$}, it holds that $$W_{2,\ell^{\infty}}(\pi_h,\pi) = O( \frac{\beta}{\alpha}\sqrt{h\log (2d)}) \, .$$
\end{example}
\begin{proof}[Sketch of Proof]
    Given $k\in \bN$, we couple the continuous-time Langevin dynamics at stationarity $Y_t,\ kh\leq t \leq (k+1)h$ and the discrete-time iterates $X_{kh}$ in \eqref{eqn-OLMC} so that they share the same Brownian motion. Introducing an auxiliary random variable
       $\overline{Y}_{(k+1)h} = Y_{kh} - h \nabla V(Y_{kh}) + \sqrt{2}(B_{(k+1)h} - B_{kh})$
   and using the triangle inequality then lead to 
   \begin{equation*}
       \sqrt{\bE[|X_{(k+1)h} - Y_{(k+1)h}|_{\infty}^2]} \leq \underbrace{\sqrt{\bE[|X_{(k+1)h} - \overline{Y}_{(k+1)h}|_{\infty}^2]}}_{(a)} + \underbrace{\sqrt{\bE[| \overline{Y}_{(k+1)h} - Y_{(k+1)h}|_{\infty}^2]}}_{(b)}\, .
   \end{equation*}
   Here $(b)$ is the one-step discretization error which is bounded by $O( \beta h^{3/2}\sqrt{\log (2d)})$ given \edit{$h \leq 1/\beta$}. And $(a) = \sqrt{\bE[|X_{kh} - Y_{kh} - h(\nabla V(X_{kh}) - \nabla V(Y_{kh}))|_{\infty}^2]} \leq (1-\alpha h)\sqrt{\bE[|X_{kh} - Y_{kh}|_{\infty}^2]} \leq \exp(-\alpha h)\sqrt{\bE[|X_{kh} - Y_{kh}|_{\infty}^2]}$ where we apply the strong convexity of $V_i$ to bound each coordinate of the vector yielding the final $\ell^\infty$ norm bound. This shows that there is a one-step contraction in the $\ell^\infty$ norm. We then couple the distribution of $X_{kh}$ and $Y_{kh}$ such that $\sqrt{\bE[|X_{kh} - Y_{kh}|_{\infty}^2]} = W_{2,\ell^{\infty}}(\rho_{kh},\pi)$. Using the definition of the $W_{2,\ell^{\infty}}$ norm, we get
   \[W_{2,\ell^{\infty}}(\rho_{(k+1)h},\pi) \leq \exp(-\alpha h)W_{2,\ell^{\infty}}(\rho_{kh},\pi) + O( \beta h^{3/2}\sqrt{\log (2d)})\, . \]
   Iterating this inequality implies $W_{2,\ell^{\infty}}(\pi_h,\pi) = O( \frac{\beta}{\alpha}\sqrt{h\log (2d)})$. A detailed proof can be found in Appendix \ref{appendix-proof-product-measure}.
\end{proof}
Example \ref{intro-product-measure-example} shows that the $W_{2,\ell^{\infty}}$ bias scales only as the square root of $\log (2d)$; that is, the bias is nearly independent of the dimension. The key in the proof for Example \ref{intro-product-measure-example} is the one-step contraction property in the $\ell^\infty$ norm, which relies on the structure of the product measures. This property does not hold for general $\pi$. On the other hand, we know that for Gaussian distributions, we have an explicit formula for the law of iterates in the algorithm and thus the biased distribution $\pi_h$, so we can investigate the $W_{2,\ell^{\infty}}$ bias directly without concerning the one-step iteration property; see the following Example \ref{intro-gaussian-example} with proof in Appendix \ref{appendix-proof-Gaussian-case}.
\begin{example}
    \label{intro-gaussian-example}
    Consider $\pi \propto \exp(-V)$ and $V(x) = \frac{1}{2}(x-m)^T\Sigma^{-1}(x-m)$ where $m \in \bR^d$ and $\alpha I\preceq \Sigma^{-1} \preceq \beta I$.  Then, for $h \leq 1/\beta$, it holds that $$W_{2,\ell^{\infty}}(\pi_h,\pi) = O( h\sqrt{\beta\log (2d)}) = O(\sqrt{h\log (2d)})\, .$$
\end{example}

Again, Example \ref{intro-gaussian-example} shows that the $W_{2,\ell^{\infty}}$ bias is nearly independent of the dimension. Using the property of the $W_{2,\ell^{\infty}}$ metric, this further implies that the $W_2$ distance between $K$-marginals of $\pi_h$ and $\pi$ is bounded by $O(\sqrt{Kh\log (2d)})$, nearly independent of $d$. We can interpret this as the overall bias being nearly \textit{delocalized}
accross all one-dimensional marginals. A step size of $O(1/K)$, up to logarithmic terms, suffices for a bounded bias and error in $K$-marginals for any dimension; the iteration complexity scales with $K$ and is also nearly independent of $d$.

On the other hand, a simple example shows that delocalization of bias does not always hold:
\begin{example}
\label{intro-negative-example}
    Consider $\pi = \rho^{\otimes d}$ where $\rho$ is a one-dimensional centered log concave distribution for which the biased distribution $\rho_h$ obtained by the corresponding one dimensional unadjusted Langevin has nonzero mean\footnote{For example, Proposition \ref{prop-formula} can be used to show that $\rho$ can be taken to be a univariate mixture of Gaussians $p\cN(\mu_1, 1) + (1-p) \cN(\mu_2, 1)$ with $p \mu_1 + (1-p)\mu_2 = 0$, $0 < p < 1/2$, and $|\mu_1-\mu_2|$ sufficiently small.}, so that their mean differs by $\delta > 0$. Consider the observable $f(x) = \frac{1}{\sqrt{d}}\sum_{i=1}^d x^{(i)}$ for $x \in \bR^d$. It holds that
    \[ \left|\int f(\pi-\pi_h)\right| = \sqrt{d}\delta\, . \]
    Now, consider the rotation matrix $Q$ which satisfies $(Qx)^{(1)} = \frac{1}{\sqrt{d}}\sum_{i=1}^d x^{(i)}$. Let $\tilde{\pi} = Q\# \pi$. We have $\tilde{\pi}_h = Q\# \pi_h$. Consequently, it holds that
    \[ \left|\int x^{(1)} (\tilde{\pi} - \tilde{\pi}_h)\right| =  \left|\int f(\pi-\pi_h)\right| = \sqrt{d}\delta \, .\]
    We then find $W_{2,\ell^{\infty}}(\tilde{\pi}, \tilde{\pi}_h) \geq W_{1,\ell^{\infty}}(\tilde{\pi}, \tilde{\pi}_h) \geq 
    \left|\int x^{(1)} (\tilde{\pi} - \tilde{\pi}_h)\right| = \sqrt{d}\delta$. 
\end{example}
Example \ref{intro-negative-example} shows that for the rotated product measure $\tilde{\pi}$, the $W_{2,\ell^{\infty}}$ bias is of order $\sqrt{d}$. This indicates that the bias is not \textit{delocalized}, but \textit{concentrated} on one specific dimension. Thus, the \textit{delocalization effect} of the bias over marginals is a delicate phenomenon and does not hold universally. We note that  Example \ref{intro-negative-example} is characterized by \textit{strong, dense} interactions between coordinates after the rotation.

\edit{We note that in Example~\ref{intro-negative-example}, the delocalization effect holds for the invariant distribution $\pi$ but not for the rotated $\tilde{\pi}$, revealing that the delocalization phenomenon is coordinate dependent. The coordinate-dependent $W_{2,\ell^\infty}$ metric enables the characterization of this phenomenon.}
\subsubsection{Strongly log-concave target distributions with sparse interactions} In this paper, our main result is to show that the delocalization effect holds for distributions with sparse interactions; see the following Figure \ref{fig:sparse-potential} and Theorem \ref{thm-sparse-informal}.
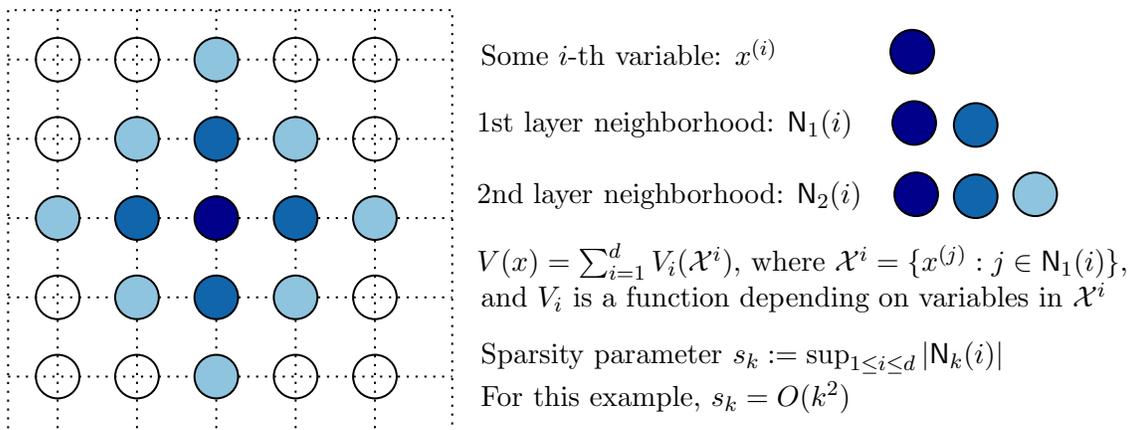
\begin{figure}[h]
    \centering
\tikzset{every picture/.style={line width=0.75pt}} 

\begin{tikzpicture}[x=0.75pt,y=0.75pt,yscale=-1,xscale=1]

\draw  [draw opacity=0][dash pattern={on 0.84pt off 2.51pt}] (37,14) -- (261,14) -- (261,227) -- (37,227) -- cycle ; \draw  [dash pattern={on 0.84pt off 2.51pt}] (62,14) -- (62,227)(102,14) -- (102,227)(142,14) -- (142,227)(182,14) -- (182,227)(222,14) -- (222,227) ; \draw  [dash pattern={on 0.84pt off 2.51pt}] (37,39) -- (261,39)(37,79) -- (261,79)(37,119) -- (261,119)(37,159) -- (261,159)(37,199) -- (261,199) ; \draw  [dash pattern={on 0.84pt off 2.51pt}]  ;
\draw  [fill={rgb, 255:red, 16; green, 101; blue, 171 }  ,fill opacity=1 ] (91,119) .. controls (91,112.92) and (95.92,108) .. (102,108) .. controls (108.08,108) and (113,112.92) .. (113,119) .. controls (113,125.08) and (108.08,130) .. (102,130) .. controls (95.92,130) and (91,125.08) .. (91,119) -- cycle ;
\draw   (91,39) .. controls (91,32.92) and (95.92,28) .. (102,28) .. controls (108.08,28) and (113,32.92) .. (113,39) .. controls (113,45.08) and (108.08,50) .. (102,50) .. controls (95.92,50) and (91,45.08) .. (91,39) -- cycle ;
\draw  [fill={rgb, 255:red, 16; green, 101; blue, 171 }  ,fill opacity=1 ] (131,79) .. controls (131,72.92) and (135.92,68) .. (142,68) .. controls (148.08,68) and (153,72.92) .. (153,79) .. controls (153,85.08) and (148.08,90) .. (142,90) .. controls (135.92,90) and (131,85.08) .. (131,79) -- cycle ;
\draw  [fill={rgb, 255:red, 142; green, 196; blue, 222 }  ,fill opacity=1 ] (91,79) .. controls (91,72.92) and (95.92,68) .. (102,68) .. controls (108.08,68) and (113,72.92) .. (113,79) .. controls (113,85.08) and (108.08,90) .. (102,90) .. controls (95.92,90) and (91,85.08) .. (91,79) -- cycle ;
\draw  [fill={rgb, 255:red, 142; green, 196; blue, 222 }  ,fill opacity=1 ] (51,119) .. controls (51,112.92) and (55.92,108) .. (62,108) .. controls (68.08,108) and (73,112.92) .. (73,119) .. controls (73,125.08) and (68.08,130) .. (62,130) .. controls (55.92,130) and (51,125.08) .. (51,119) -- cycle ;
\draw   (51,79) .. controls (51,72.92) and (55.92,68) .. (62,68) .. controls (68.08,68) and (73,72.92) .. (73,79) .. controls (73,85.08) and (68.08,90) .. (62,90) .. controls (55.92,90) and (51,85.08) .. (51,79) -- cycle ;
\draw   (51,39) .. controls (51,32.92) and (55.92,28) .. (62,28) .. controls (68.08,28) and (73,32.92) .. (73,39) .. controls (73,45.08) and (68.08,50) .. (62,50) .. controls (55.92,50) and (51,45.08) .. (51,39) -- cycle ;
\draw   (211,39) .. controls (211,32.92) and (215.92,28) .. (222,28) .. controls (228.08,28) and (233,32.92) .. (233,39) .. controls (233,45.08) and (228.08,50) .. (222,50) .. controls (215.92,50) and (211,45.08) .. (211,39) -- cycle ;
\draw   (171,39) .. controls (171,32.92) and (175.92,28) .. (182,28) .. controls (188.08,28) and (193,32.92) .. (193,39) .. controls (193,45.08) and (188.08,50) .. (182,50) .. controls (175.92,50) and (171,45.08) .. (171,39) -- cycle ;
\draw  [fill={rgb, 255:red, 142; green, 196; blue, 222 }  ,fill opacity=1 ] (131,39) .. controls (131,32.92) and (135.92,28) .. (142,28) .. controls (148.08,28) and (153,32.92) .. (153,39) .. controls (153,45.08) and (148.08,50) .. (142,50) .. controls (135.92,50) and (131,45.08) .. (131,39) -- cycle ;
\draw   (211,79) .. controls (211,72.92) and (215.92,68) .. (222,68) .. controls (228.08,68) and (233,72.92) .. (233,79) .. controls (233,85.08) and (228.08,90) .. (222,90) .. controls (215.92,90) and (211,85.08) .. (211,79) -- cycle ;
\draw  [fill={rgb, 255:red, 142; green, 196; blue, 222 }  ,fill opacity=1 ] (211,119) .. controls (211,112.92) and (215.92,108) .. (222,108) .. controls (228.08,108) and (233,112.92) .. (233,119) .. controls (233,125.08) and (228.08,130) .. (222,130) .. controls (215.92,130) and (211,125.08) .. (211,119) -- cycle ;
\draw   (211,159) .. controls (211,152.92) and (215.92,148) .. (222,148) .. controls (228.08,148) and (233,152.92) .. (233,159) .. controls (233,165.08) and (228.08,170) .. (222,170) .. controls (215.92,170) and (211,165.08) .. (211,159) -- cycle ;
\draw  [fill={rgb, 255:red, 142; green, 196; blue, 222 }  ,fill opacity=1 ] (171,159) .. controls (171,152.92) and (175.92,148) .. (182,148) .. controls (188.08,148) and (193,152.92) .. (193,159) .. controls (193,165.08) and (188.08,170) .. (182,170) .. controls (175.92,170) and (171,165.08) .. (171,159) -- cycle ;
\draw  [fill={rgb, 255:red, 16; green, 101; blue, 171 }  ,fill opacity=1 ] (131,159) .. controls (131,152.92) and (135.92,148) .. (142,148) .. controls (148.08,148) and (153,152.92) .. (153,159) .. controls (153,165.08) and (148.08,170) .. (142,170) .. controls (135.92,170) and (131,165.08) .. (131,159) -- cycle ;
\draw  [fill={rgb, 255:red, 142; green, 196; blue, 222 }  ,fill opacity=1 ] (91,159) .. controls (91,152.92) and (95.92,148) .. (102,148) .. controls (108.08,148) and (113,152.92) .. (113,159) .. controls (113,165.08) and (108.08,170) .. (102,170) .. controls (95.92,170) and (91,165.08) .. (91,159) -- cycle ;
\draw  [fill={rgb, 255:red, 16; green, 101; blue, 171 }  ,fill opacity=1 ] (171,119) .. controls (171,112.92) and (175.92,108) .. (182,108) .. controls (188.08,108) and (193,112.92) .. (193,119) .. controls (193,125.08) and (188.08,130) .. (182,130) .. controls (175.92,130) and (171,125.08) .. (171,119) -- cycle ;
\draw  [fill={rgb, 255:red, 0; green, 0; blue, 139 }  ,fill opacity=1 ] (131,119) .. controls (131,112.92) and (135.92,108) .. (142,108) .. controls (148.08,108) and (153,112.92) .. (153,119) .. controls (153,125.08) and (148.08,130) .. (142,130) .. controls (135.92,130) and (131,125.08) .. (131,119) -- cycle ;
\draw  [fill={rgb, 255:red, 142; green, 196; blue, 222 }  ,fill opacity=1 ] (171,79) .. controls (171,72.92) and (175.92,68) .. (182,68) .. controls (188.08,68) and (193,72.92) .. (193,79) .. controls (193,85.08) and (188.08,90) .. (182,90) .. controls (175.92,90) and (171,85.08) .. (171,79) -- cycle ;
\draw   (211,199) .. controls (211,192.92) and (215.92,188) .. (222,188) .. controls (228.08,188) and (233,192.92) .. (233,199) .. controls (233,205.08) and (228.08,210) .. (222,210) .. controls (215.92,210) and (211,205.08) .. (211,199) -- cycle ;
\draw   (171,199) .. controls (171,192.92) and (175.92,188) .. (182,188) .. controls (188.08,188) and (193,192.92) .. (193,199) .. controls (193,205.08) and (188.08,210) .. (182,210) .. controls (175.92,210) and (171,205.08) .. (171,199) -- cycle ;
\draw  [fill={rgb, 255:red, 142; green, 196; blue, 222 }  ,fill opacity=1 ] (131,199) .. controls (131,192.92) and (135.92,188) .. (142,188) .. controls (148.08,188) and (153,192.92) .. (153,199) .. controls (153,205.08) and (148.08,210) .. (142,210) .. controls (135.92,210) and (131,205.08) .. (131,199) -- cycle ;
\draw   (91,199) .. controls (91,192.92) and (95.92,188) .. (102,188) .. controls (108.08,188) and (113,192.92) .. (113,199) .. controls (113,205.08) and (108.08,210) .. (102,210) .. controls (95.92,210) and (91,205.08) .. (91,199) -- cycle ;
\draw   (51,199) .. controls (51,192.92) and (55.92,188) .. (62,188) .. controls (68.08,188) and (73,192.92) .. (73,199) .. controls (73,205.08) and (68.08,210) .. (62,210) .. controls (55.92,210) and (51,205.08) .. (51,199) -- cycle ;
\draw   (51,159) .. controls (51,152.92) and (55.92,148) .. (62,148) .. controls (68.08,148) and (73,152.92) .. (73,159) .. controls (73,165.08) and (68.08,170) .. (62,170) .. controls (55.92,170) and (51,165.08) .. (51,159) -- cycle ;
\draw  [fill={rgb, 255:red, 142; green, 196; blue, 222 }  ,fill opacity=1 ] (544,107) .. controls (544,100.92) and (548.92,96) .. (555,96) .. controls (561.08,96) and (566,100.92) .. (566,107) .. controls (566,113.08) and (561.08,118) .. (555,118) .. controls (548.92,118) and (544,113.08) .. (544,107) -- cycle ;
\draw  [fill={rgb, 255:red, 0; green, 0; blue, 139 }  ,fill opacity=1 ] (482,35) .. controls (482,28.92) and (486.92,24) .. (493,24) .. controls (499.08,24) and (504,28.92) .. (504,35) .. controls (504,41.08) and (499.08,46) .. (493,46) .. controls (486.92,46) and (482,41.08) .. (482,35) -- cycle ;
\draw  [fill={rgb, 255:red, 0; green, 0; blue, 139 }  ,fill opacity=1 ] (483,71) .. controls (483,64.92) and (487.92,60) .. (494,60) .. controls (500.08,60) and (505,64.92) .. (505,71) .. controls (505,77.08) and (500.08,82) .. (494,82) .. controls (487.92,82) and (483,77.08) .. (483,71) -- cycle ;
\draw  [fill={rgb, 255:red, 0; green, 0; blue, 139 }  ,fill opacity=1 ] (484,107) .. controls (484,100.92) and (488.92,96) .. (495,96) .. controls (501.08,96) and (506,100.92) .. (506,107) .. controls (506,113.08) and (501.08,118) .. (495,118) .. controls (488.92,118) and (484,113.08) .. (484,107) -- cycle ;
\draw  [fill={rgb, 255:red, 16; green, 101; blue, 171 }  ,fill opacity=1 ] (514,72) .. controls (514,65.92) and (518.92,61) .. (525,61) .. controls (531.08,61) and (536,65.92) .. (536,72) .. controls (536,78.08) and (531.08,83) .. (525,83) .. controls (518.92,83) and (514,78.08) .. (514,72) -- cycle ;
\draw  [fill={rgb, 255:red, 16; green, 101; blue, 171 }  ,fill opacity=1 ] (514,108) .. controls (514,101.92) and (518.92,97) .. (525,97) .. controls (531.08,97) and (536,101.92) .. (536,108) .. controls (536,114.08) and (531.08,119) .. (525,119) .. controls (518.92,119) and (514,114.08) .. (514,108) -- cycle ;

\draw (274,27) node [anchor=north west][inner sep=0.75pt]   [align=left] {Some $i$-th variable: $x^{(i)}$};
\draw (272,63) node [anchor=north west][inner sep=0.75pt]   [align=left] {1st layer neighborhood: $\sfN_1(i)$};
\draw (272,99) node [anchor=north west][inner sep=0.75pt]   [align=left] {2nd layer neighborhood: $\sfN_2(i)$};
\draw (272,130) node [anchor=north west][inner sep=0.75pt]   [align=left] {$V(x) = \sum_{i=1}^d V_i(\mathcal{X}^i)$, where $\mathcal{X}^i = \{x^{(j)}: j \in \mathsf{N}_1(i)\}$,};
\draw (274,150) node [anchor=north west][inner sep=0.75pt]   [align=left] {and $V_i$ is a function depending on variables in $\mathcal{X}^i$};
\draw (274,180) node [anchor=north west][inner sep=0.75pt]   [align=left] {Sparsity parameter $s_k := \sup_{1\leq i\leq d}|\sfN_{2k}(i)|$};

\draw (274,200) node [anchor=north west][inner sep=0.75pt]   [align=left] {For this example, $s_k = O(k^2)$};

\end{tikzpicture}
    \caption{Illustration of a potential $V(x),x\in \bR^d$ with sparse interactions.}
    \label{fig:sparse-potential}
\end{figure}

\begin{theorem}[informal]
\label{thm-sparse-informal}
    Let Assumption \ref{assumption-V-log-concave} hold. Suppose $V$ further satisfies the sparsity conditions illustrated in Figure \ref{fig:sparse-potential} and rigorously formulated in Section \ref{sec-Sparse-graphical-models} below with the sparsity parameter $s_k \leq C(k+1)^n$, then for $h\leq 1/\beta$,
    \[W_{2,\ell^{\infty}}(\pi_h,\pi) \leq \sqrt{h\log(2d)}\min\left\{\left(O\big(\frac{\beta}{\alpha}\log(2d)\big)\right)^{\frac{n}{2}+1},\  O\left(\frac{\beta}{\alpha}\sqrt{d}\right)\right\}\, .\]
\end{theorem}
The precise statement of the theorem is in Section \ref{sec-Strongly Logconcave Distributions with Sparse Interactions}. The key technical challenge in extending the delocalization of bias result beyond product measures and Gaussian measures is the lack of one-step contraction in the $W_{2,\ell^\infty}$ metric and the lack of explicit formulas of the solutions. We employ a novel multi-step coupling argument to derive a multi-step contraction result and crucially use the sparsity of the potential to control the accumulated errors over these steps to prove the theorem. 

The theorem implies that for such a sparse potential, the $W_{2,\ell^{\infty}}$ bias is nearly independent of $d$, up to logarithmic terms and assuming the condition number $\beta/\alpha$ is also nearly independent of $d$. As a consequence, by taking $h \sim 1/K$, the iteration complexity scales with $K$ to ensure a desired $W_2$ error for all $K$-marginals, up to logarithmic terms on $d$. Thus, for these potentials, the unadjusted Langevin algorithm proves more scalable than its unbiased variants when the quantities of interest are low-dimensional marginals. We note that these sparse potential illustrated in Figure \ref{fig:sparse-potential} can arise in physical applications and Bayesian inference; see Section \ref{sec-example-graphical-models}. 

 Generalizing the delocalization of bias effect beyond sparse potentials can provide further insights on the behavior of the algorithm in high dimensions. We also approach this through an alternative asymptotic perspective, in the spirit of methods using Taylor's expansion and the Poisson equation \cite{talay1990second,talay1990expansion,bally1996law,mattingly2010convergence}. We derive explicit formulas for the bias of observables in first orders and demonstrate how the delocalization of bias, in the context of observables, can hold in a wider generality.
\subsection{Organization of this paper}
Section \ref{sec-Strongly Logconcave Distributions with Sparse Interactions} presents our main result on sparse potentials. In Section \ref{sec-Discussion and Generalizations}, we discuss generalizations of our result through asymptotic arguments. We conclude the paper in Section \ref{sec-conclusions}. All proofs are deferred to the appendices.
\subsection{Notations}
We write $A = O(B)$ or $A \lesssim B$ to mean that there exists a constant $C$ independent of $\alpha, \beta, h, d$ such that $A \leq CB$. On the other hand, $A = \Omega(B)$ or $A \gtrsim B$ means that there exists a constant $C$ independent of $\alpha, \beta, h, d$ such that $A \geq CB$. We use $|\cdot|_{\infty}$ and $|\cdot|_{2}$ to represent the $\ell^{\infty}$ and $\ell^2$ norms for vectors and matrices. When applied to matrices, they stand for the corresponding operator norms, as in Section \ref{sec-a-new-metric-measure-convergence}. We use $\preceq$ for the Loewner order such that if $M \preceq N$ where $M,N$ are symmetric matrices, then the matrix $N-M$ is  positive semi-definite.

\section{Strongly Log-concave Distributions with Sparse Interactions}
\label{sec-Strongly Logconcave Distributions with Sparse Interactions}
In this section, we begin by mathematically defining the potential with sparse interactions (as illustrated in Figure \ref{fig:sparse-potential}) in subsection \ref{sec-Sparse-graphical-models}. Subsection \ref{sec-Convergence-bound-sparse-model} presents the convergence bound. Finally, in subsection \ref{sec-example-graphical-models}, we discuss examples of sparse potentials that satisfy the assumptions of our theorem.
\subsection{Sparse graphical models}
\label{sec-Sparse-graphical-models}
Consider an undirected graph $G$ with $d$ nodes, labeled by $1\leq i \leq d$. If there is an edge between node $i$ and $j$, we write $i \sim j$, meaning that the two nodes are connected. We denote by $\sfN(i)$ the \textit{neighborhood} of the node $i$, which is a set of nodes connected to $i$. Here the neighborhood relationship is symmetric, and without loss of generality, we assume $i \sim i$ for all $1\leq i \leq d$. Furthermore, we define recursively that $\sfN_k(i) = \{1\leq j \leq d: \exists\ \ell \in \sfN_{k-1}(i), \text{ such that } j \sim \ell \}$, for $k \geq 2$. Here $\sfN_1(i) := \sfN(i)$. Denote the cardinality of the set $\sfN_k(i)$ by $|\sfN_k(i)|$ and \edit{let $\max_{1\leq i \leq d} |\sfN_{2k}(i)| = s_k$.} We also write $s = s_1$. 

\begin{assumption}
\label{assume-sparse-potential-function}
    Let the potential $V$ be of the form 
\begin{equation}
\label{eqn-sparse-potential-function}
    V(x) = \sum_{i=1}^d V_i (\mathcal{X}^i)
\end{equation}
where $\mathcal{X}^i = \{x^{(j)}: j \in \mathsf{N}(i)\}$ and $V_i$ is a function depending on variables in the set $\mathcal{X}^i$. 
\end{assumption}
We can also understand $G$ as a \textit{factor graph}, where $F =\{\mathcal{X}^i, 1\leq i \leq d\}$ represents the \textit{factor vertices}. By definition, the potential $V$ described above exhibits sparse interaction when the graph is sparse. \edit{We observe that, under Assumption \ref{assume-sparse-potential-function},
 the number of non-zero entries in each column of $\nabla^2 V(x)$ is bounded by $s_1$. More generally, when multiplying $k$ such Hessians, the number of non-zero entries per column is bounded by $s_k$, a fact that plays a key role in our subsequent analysis.}

\subsection{Convergence bound for unadjusted Langevin}
\label{sec-Convergence-bound-sparse-model}
Below we present the bias and convergence bounds in terms of the $W_{2,\ell^{\infty}}$ metric. The sparsity parameters $s_k$ will play important roles in our bound.
\begin{theorem}
\label{thm-OLD-bias-under-sparsity}
   Let Assumptions \ref{assumption-V-log-concave} and \ref{assume-sparse-potential-function} hold. Assume $0 \leq h \leq 1/\beta$ and denote $q = \exp(-h\alpha)$. Let $r_i = \lceil e^2 ih\beta + \log \sqrt{d}\rceil$ for $i \in \bN$. Then, for any $N \in \bN$ such that $2q^N\sqrt{s_{r_N}}<1$, we have the bound
    \begin{equation}
    \label{eqn-bound-general}
    \begin{aligned}
        W_{2,\ell^{\infty}}(\pi_h,\pi) &\leq  \mathsf{bias}(N, q, \beta, h)\\
        & := \frac{2\beta\sum_{i=1}^N q^{i-1}\sqrt{s_{r_i}}}{1-2q^N\sqrt{s_{r_N}}}\left(h^2\sqrt{\bE_\pi [|\nabla V|_{\infty}^2]} + 3h^{3/2}\sqrt{\log(2d)}\right)\, .
    \end{aligned}
    \end{equation}
    Moreover, the following convergence bound holds for any $k \in \bN$:
\begin{equation}W_{2,\ell^{\infty}}(\rho_{kh}, \pi) 
\leq q^k W_2(\rho_0, \pi_h) + \mathsf{bias}(N, q, \beta, h)\, .
\end{equation}
\end{theorem}
\begin{proof}[Sketch of proof]
    Using the same coupling argument and notations in Example \ref{intro-product-measure-example}, we have
    \[\sqrt{\bE[|X_{(k+1)h} - Y_{(k+1)h}|_{\infty}^2]} \leq (a) + O( \beta h^{3/2}\sqrt{\log (2d)})\, ,\]
    where $(a)=\sqrt{\bE[|X_{kh} - Y_{kh} - h(\nabla V(X_{kh}) - \nabla V(Y_{kh}))|_{\infty}^2]}$. Now $\pi$ is not a product measure and we no longer have the contraction $(a) \leq (1-\alpha h) \sqrt{\bE[|X_{kh} - Y_{kh}|_{\infty}^2]}$. In fact, using Taylor's expansion, we can write
    \[X_{kh} - Y_{kh} - h(\nabla V(X_{kh}) - \nabla V(Y_{kh})) = H_k(X_{kh}-Y_{kh})\] where
    $H_k = I - h \int_0^1 \nabla^2 V (uX_{kh}+(1-u)Y_{kh}){\rm d}u$. Here $|H_k|_2 \leq 1-\alpha h$. When $\pi$ is a product measure, $H_k$ is a diagonal matrix and thus its $\ell^\infty$ norm equals $\ell^2$ norm so we can get the contraction property. However, in general,  $|H_k|_\infty$ can be much larger than $|H_k|_2$ so the one-step contraction fails. Specifically $|H_k|_\infty \leq \sqrt{s_1}|H_k|_2$ based on the fact that $H_k$ has only $s_1$ nonzero entries in each column. 
    As a result, we must employ a multi-step coupling argument to analyze the iterations. We now have 
\[\sqrt{\bE[|X_{(k+1)h} - Y_{(k+1)h}|_{\infty}^2]} \leq \sqrt{\bE[|H_k(X_{kh} - Y_{kh})|_{\infty}^2]} + \mathrm{error}(1)\, ,\]
where $\mathrm{error}(1) = O( \beta h^{3/2}\sqrt{\log (2d)})$.
Applying the bound again within the first term on the right hand size, yields
\[\sqrt{\bE[|X_{(k+1)h} - Y_{(k+1)h}|_{\infty}^2]}\leq \sqrt{\bE[|H_kH_{k-1}(X_{(k-1)h}-{Y}_{(k-1)h})|_{\infty}^2]} + \mathrm{error}(2)\, ,\]
where $\mathrm{error}(2)$ represents the two-step error. More generally, for any $k,N\in \bN$, we get
\[\sqrt{\bE[|X_{(k+N)h} - Y_{(k+N)h}|_{\infty}^2]}\leq \sqrt{\bE[|H_{k+N-1}H_{k+N-2}\cdots H_{k}(X_{kh}-{Y}_{kh})|_{\infty}^2]} + \mathrm{error}(N)\, .\]
Now by choosing a large $N$, we will get $|H_{k+N-1}H_{k+N-2}\cdots H_{k}|_{\infty} < 1$. In fact, a simple bound is \[|H_{k+N-1}H_{k+N-2}\cdots H_{k}|_{\infty} \leq \sqrt{d}|H_{k+N-1}H_{k+N-2}\cdots H_{k}|^2_{2} \leq \sqrt{d}\exp(-N\alpha h)\, .\]
Thus $N \sim \frac{\log d}{h}$ will lead to a contraction. In our proof, we leverage the sparsity of the potential to produce a potentially tighter bound
$|H_{k+N-1}H_{k+N-2}\cdots H_{k}|_{\infty} \leq 2q^N \sqrt{s_{r_N}}$.

    The primary challenge lies in controlling the growth of  $\mathrm{error}(N)$. A na\"{\i}ve approach would lead to polynomial dependence on $d$ since the bound may grow exponentially fast in $N$
    without contraction.
    In contrast, our analysis hinges on a sparsity bound of the propagator of the unadjusted Langevin algorithm over multiple steps; see Appendix \ref{appendix-Sparsity-of-the-propagator}. This sparsity enables us to obtain tighter bounds on these $\ell^{\infty}$ errors across multiple iterations. In fact, we can see that in \eqref{eqn-bound-general}, $r_i$ scales with the physical time $ih$, which justifies that $s_{r_i}$ can characterize the sparsity of interactions after $i$ steps of the algorithm. Using the sparsity bound, we finally control the accumulated discretization errors in $N$ steps by \[\mathrm{error}(N) \leq 2\beta(\sum_{i=1}^N q^{i-1}\sqrt{s_{r_i}})\left(h^2\sqrt{\bE_\pi [|\nabla V|_{\infty}^2]} + 3h^{3/2}\sqrt{\log(2d)}\right)\, .\]
With the $N$-step contraction and the bound on the accumulated discretization errors, we will get
 \[W_{2,\ell^{\infty}}(\rho_{(k+N)h},\pi) \leq 2q^N\sqrt{s_{r_N}}W_{2,\ell^{\infty}}(\rho_{kh},\pi) + \mathrm{error}(N)\, , \]
for any $k, N \in \bN$. When $2q^N\sqrt{s_{r_N}}<1$, we can let $k \to \infty$ and 
get the final bound in \eqref{eqn-bound-general}. The complete proof for Theorem \ref{thm-OLD-bias-under-sparsity} can be found in Appendix \ref{appendix-proof-logconcave distributions with sparse interactions}.
\end{proof}

Overall, Theorem \ref{thm-OLD-bias-under-sparsity} implies that the bias depends on the sparsity growth, the expectation of $|\nabla V|_{\infty}^2$, and $\log (2d)$. In the following, we state a result that provides an upper bound on the expected $|\nabla V|_{\infty}^2$. 
\begin{proposition}
\label{thm-inf-bound}
    Let Assumptions \ref{assumption-V-log-concave} hold. Then, it holds that 
   \begin{equation}
   \label{eqn-expected-inf-V-bound}
         \sqrt{\bE_\pi[|\nabla V|_{\infty}^2]} \leq \edit{2\sqrt{\beta\log (2d)}\, .}
     \end{equation}
     
\end{proposition}
\begin{proof}[Sketch of Proof]
\edit{Under the assumption, it holds that $\nabla V$ is $\sqrt{\beta}$ sub-Gaussian under $\pi$ \cite[Theorem 1.2]{altschuler2023shifted}. The expected $\ell^\infty$ norm of a sub-Gaussian vector is bounded using standard results.}
     The complete proof can be found in Appendix \ref{sec-proof-convergence-polynomial-sparsity}. 
\end{proof}

 Now, we are ready to provide one instantiation of Theorem \ref{thm-OLD-bias-under-sparsity} and Proposition \ref{thm-inf-bound} for the case that the sparsity of the neighborhood of the graphical model grows at most polynomially. We need to carefully calculate an explicit bound on the terms in \eqref{eqn-bound-general}; the detailed proof is in Appendix \ref{appendix-proof-logconcave distributions with sparse interactions}. Section \ref{sec-example-graphical-models} will provide concrete examples illustrating this case.
\begin{theorem}
\label{prop-polynomial-growth-s}
    Let Assumptions \ref{assumption-V-log-concave} and \ref{assume-sparse-potential-function} hold with $s_k \leq C(k+1)^n$ for any $k \in \bN$, where $C>0$ is a generic constant and $n \in \bN$. Then for \edit{$h \leq 1/\beta$}, we have
    \begin{equation}
    \label{eqn-bound-bias-sparsity-example}
        W_{2,\ell^{\infty}}(\pi_{h},\pi) \leq \sqrt{h\log(2d)}\min\left\{\left(O\big(\frac{\beta}{\alpha}\log(2d)\big)\right)^{\frac{n}{2}+1},\  O\left(\frac{\beta}{\alpha}\sqrt{d}\right)\right\}\, .
    \end{equation}
    Moreover, it holds that
    \begin{equation}W_{2,\ell^{\infty}}(\rho_{kh}, \pi) 
\leq q^k W_2(\rho_0, \pi_h) + \sqrt{h\log(2d)}\min\left\{\left(O\big(\frac{\beta}{\alpha}\log(2d)\big)\right)^{\frac{n}{2}+1},\  O\left(\frac{\beta}{\alpha}\sqrt{d}\right)\right\}\, .
\end{equation}

\end{theorem}
In Theorem \ref{prop-polynomial-growth-s}, we express $\left(O\big(\frac{\beta}{\alpha}\log(2d)\big)\right)^{\frac{n}{2}+1}$ in this manner to ensure that the hidden constant is also independent of $n$. In the worst-case scenario, where there is no sparsity present at all, $n$ scales with $\log d$.

Theorem \ref{prop-polynomial-growth-s} implies that the $W_{2,\ell^{\infty}}$ bias is nearly independent of $d$, up to logarithmic terms, if we assume $n$ is independent of $d$ and the condition number $\beta/\alpha$ is also nearly independent of $d$. In such case, by taking $h \sim 1/K$, the iteration complexity scales with $K$ to ensure a bounded $W_2$ error for all $K$-marginals, up to logarithmic terms on $d$, per the discussion in Section \ref{sec-intro-main-result}.

\begin{newremark}
    The current upper bound in Theorem \ref{prop-polynomial-growth-s} involves $\sqrt{h}$. This is in alignment with the $W_2$ bound proved in the literature (under Assumption \ref{assumption-V-log-concave}) based on coupling arguments where such square root dependence manifests. The second part of the bound, $\sqrt{h\log(2d)} O\left(\frac{\beta}{\alpha}\sqrt{d}\right)$, is consistent, up to logarithmic terms in $d$, with the established result for the $W_2$ bound. In fact, using the $W_2$ bound, we have that $W_{2,\ell^{\infty}}(\pi_h,\pi) \leq W_2(\pi_h,\pi) = O(\frac{\beta}{\alpha}\sqrt{dh})$.

    \edit{It may appear that in \eqref{eqn-bound-bias-sparsity-example}, the first bound is always superior to the second in terms of dimensional scaling. However, the second bound exhibits better dependence on the condition number $\beta/\alpha$. When $n \sim \log d$, corresponding to the case where sparsity is absent, the second bound becomes preferable.  This analysis also suggests that the dependence of the first bound on the condition number may be further improved through more refined analysis. We note that for Gaussians, Example~\ref{intro-gaussian-example} shows that there is no dependence on condition numbers, which provides further evidence that improving the condition number dependence is possible.}
\end{newremark}

\subsection{Examples of sparse graphical models}
\label{sec-example-graphical-models}
In this subsection, we provide examples that satisfy Assumptions \ref{assumption-V-log-concave} and \ref{assume-sparse-potential-function}. For the bound in Theorem \ref{prop-polynomial-growth-s} to be non-trivial, the ratio $\beta/\alpha$ should not grow too rapidly with the dimension. The simplest case would be when $\alpha$ and $\beta$ are fixed constants, independent of the dimension.

\edit{We show that there exist non-trivial examples satisfying our assumptions. Specifically, for given constants $\alpha$ and $\beta$ with $\alpha \leq \beta$, there exists a potential function $V$ defined for any dimension $d \in \mathbb{N}$ such that $\alpha I \preceq \nabla^2 V(x) \preceq \beta I$ for all $x \in \mathbb{R}^d$, and $s_k$ grows at most polynomially in $k$, as required in Theorem \ref{prop-polynomial-growth-s}.
A natural class of such potentials arises from Laplacians of sparse graphs, as illustrated in the following example.}
\begin{example}
\label{example-sparse-potential}
   We provide an example of $V:\bR^d \to \bR$ for which $\alpha I \preceq \nabla^2 V(x) \preceq \beta I$ at any $x \in \bR^d$ with $\alpha, \beta$ independent of $d$, and $s_k = \min\{ 2k+1, d\}$.
   
   Consider the following matrix 
   \begin{equation*}
       \begin{bmatrix}
2+\lambda(x) & -1 & 0 & 0 & \cdots & 0 \\
-1 & 2+\lambda(x) & -1 & 0 & \cdots & 0 \\
0 & -1 & 2+\lambda(x) & -1 & \cdots & 0 \\
\vdots & \vdots & \vdots & \vdots & \ddots & \vdots \\
0 & 0 & \cdots & -1 & 2+\lambda(x) & -1 \\
0 & 0 & \cdots & 0 & -1 & 2+\lambda(x) \\
\end{bmatrix} \in \bR^{d\times d}\, .
   \end{equation*}
   This matrix can be seen as $-\Delta + \lambda(x)I$, where $\Delta$ is the discretization of the 1D Laplace operator with homogeneous Dirichlet boundary conditions. 
   This operator satisfies $0 \preceq -\Delta \preceq 4 I$.
   Therefore, if $V$ is a potential such that $\nabla^2 V(x)$ is equal to the above matrix and $\lambda$ is a bounded function satisfying $\min_{x \in \bR^d} \lambda(x) = \alpha > 0$, then we have $s_k = \max\{2k+1,d\}$ and $\alpha I \preceq \nabla^2 V(x) \preceq \beta I$ with $\beta = 4 + \max_{x\in \bR^d}\lambda(x)$, in any dimension $d$. Thus this model satisfies the assumption in Theorem \ref{prop-polynomial-growth-s} with $n=1$. 
\end{example}
More generally, consider a sparse, bounded-degree graph $G$ on the variables $x^{(1)}, \dots, x^{(d)}$ and potentials of the form $V(x) = \sum_{i=1}^d V_i(x^{(i)})  + \frac 12 \sum_{i, j: (i, j) \in E(G)} (x^{(i)} - x^{(j)})^2$, where $V_1, \dots, V_d$ are arbitrary smooth and strongly convex functions and $E(G)$ is the edge set of $G$.
The resulting probability measures are log-concave perturbations of a Gaussian free field.
Then $V$ will satisfy Assumptions \ref{assumption-V-log-concave} and \ref{assume-sparse-potential-function} with $\alpha$ and $\beta$ independent of the dimension, analogous to Example \ref{example-sparse-potential}.
An important special case is when $G$ is a subset of the lattice $\bZ^n$, in which case the sparsity parameter satisfies $s_k \leq C(k+1)^n.$

These examples can arise in physical systems where there are local interactions and a restoring force for each state which plays the role of $V_i(x^{(i)})$, or in Bayesian inverse problems where such $V_i(x^{(i)})$ terms can come from priors.

\section{Discussions and Generalization with Asymptotic Arguments}
\label{sec-Discussion and Generalizations}
In this section, we discuss generalizations of the results beyond log-concave distributions with sparse interactions. While our analysis requires strict sparsity, where most interactions are zero, we anticipate that the analysis can be extended to cases where all but a small number of interactions are weak. Moreover, our sparsity assumptions treat different coordinates equally, and it is of interest to consider heterogeneous sparse models where some coordinates have more, but potentially weaker, connections with others. The study of these scenarios is left for future endeavors.

Alternatively, we can approach the question from a different perspective, employing asymptotic analysis. More precisely, we adopt the Poisson argument and utilize the generator of the Markov process to derive an asymptotic formula for the bias of certain observables. By doing so, we can gain additional insights into the algorithm's bias.

Recall our definition $V(x) = -\log \pi(x)$. Let $\cL$ and $\cL_h$ be the generators of the continuous Langevin dynamics and the unadjusted Langevin algorithm, respectively. By definition, $\cL u = \nabla \log \pi \cdot \nabla u + \Delta u$ and
\[\cL_h u(x) = \frac{\bE[u(x+h \nabla \log \pi(x)+ \sqrt{2h}\xi)] - u(x)}{h}\, . \]
Below we calculate the first order formula for the bias of an observable $f$. Here, we use the notation $A=o(h)$ to mean that the term $A$ is a higher order term than $h$ as $h \to 0$.
\begin{proposition} Without loss of generality, assume $\int f \pi = 0$. \edit{Then, under  regularity assumptions provided in Appendix \ref{appendix-A Formula for the Asymptotic Bias}}, it holds that
\label{prop-formula}
    \begin{equation}
        \int f\pi - \int f\pi_h = \frac{1}{4}h \int (-2\Delta f + |\nabla \log \pi|^2_2f)\pi + o(h)\, .
    \end{equation}
    Moreover, we also have the following formula:
     \begin{equation}
        \int f\pi - \int f\pi_h = -\frac{1}{4}h \int (\Delta f + f\Delta \log \pi)\pi + o(h)\, .
    \end{equation}
\end{proposition}
The proof of the formula can be found in Appendix \ref{appendix-A Formula for the Asymptotic Bias}.

Let us make some observations regarding Proposition \ref{prop-formula}. If $\pi$ is Gaussian, then $\int f (\Delta \log \pi) \pi = 0$ since $\int f\pi = 0$. We have
\[\int f\pi - \int f\pi_h = -\frac{1}{4}h\int (\Delta f) \pi + o(h)\, .\]
This means that if $f$ is a linear observable, then the first-order term of the bias is zero. Moreover, if $f$ depends only on a small number of coordinates of $x \in \bR^d$, then the integral $\int (\Delta f) \pi$ will scale with that number, rather than $d$, because the entire integral will rely only on the marginal distribution of $\pi$ at these coordinates. Thus, to leading order, the bias of the observable is delocalized across dimensions.

In fact, the above argument can be further generalized to $\pi$ that is a perturbation of Gaussians; see the following Proposition \ref{prop-gaussian-perturbation}, and its proof in Appendix \ref{appendix-A Formula for the Asymptotic Bias}.
\begin{proposition}
\label{prop-gaussian-perturbation}
    Let $\pi \propto \exp(-V(x)) \propto \cN(x;m,\Sigma)\exp(-U(x))$ where $\cN(x;m,\Sigma)$ is the density of a Gaussian with mean $m$ and covariance $\Sigma$. Assume $\pi$ is centered, and 
    \edit{$\int (\Delta U)^2 \pi - (\int \Delta U \pi)^2 \leq C_0$} for some constant $C_0$ independent of $d$. Suppose $f$ depends only on $K$ coordinates of $x \in \bR^d$, then the first-order term of the bias depends on $K$, not $d$. In particular, for the observable $f(x) = \sum_{k=1}^K x^{(k)}$, under the additional assumption that $\pi$ satisfies a Poincar\'e inequality with a positive dimension-independent constant, we have
    \[\left|\int f\pi - \int f\pi_h\right| \leq C_1 \sqrt{K}h + o(h)\, , \]
    where $C_1$ is a constant independent of $d$.
\end{proposition}
Proposition \ref{prop-gaussian-perturbation} shows that the delocalization of bias for observables holds if $\pi$ is a perturbation of a Gaussian distribution, suggesting that the  delocalization effect may hold in significant generality. \edit{Note, however, the asymptotic arguments in this section do not provide insight into the behavior of the higher order contributions to the bias.} On the other hand, the asymptotic formula is applicable beyond log-concave distributions and can be applied to study the bias of any observable of interest.

Proposition \ref{prop-gaussian-perturbation} also suggests that the Gaussian part may not have a significant effect on the bias of the observable $f$. Recall that our $W_{2,\ell^{\infty}}$ bounds apply to Gaussian distributions (which can have dense interactions) and log-concave distributions with sparse interactions. It is natural to inquire whether our $W_{2,\ell^{\infty}}$ analysis can be extended to the case of a distribution that is the product of a dense Gaussian and a sparse log-concave distribution.

\section{Conclusions}
\label{sec-conclusions}
In this article, we studied the convergence of unadjusted Langevin algorithms in high dimensions. For strongly log-concave distributions, existing results showed that the iteration complexity scales proportionally to $d$ or $\sqrt{d}$ in order to achieve a desired error in the $W_2$ metric for any dimension $d$. We demonstrate that for Gaussian distributions or distributions with certain sparsity structures, a constant number of iterations, up to some logarithmic terms in $d$, suffices to achieve a bounded error in the $W_{2, \ell^{\infty}}$ metric. Consequently, a number of iterations proportional to $K$ (up to some $\log d$ terms) can achieve a bounded $W_2$ error for all $K$-marginals. This result implies that even in extremely high-dimensional settings, unadjusted Langevin algorithms can still be highly scalable if the quantities of interest depend only on low-dimensional marginals. We note that this desirable property is not satisfied for unbiased schemes such as the MALA or proximal samplers, where the constraint on the step size necessitates poor scaling in $d$.

The delocalization of bias effect is rigorously shown in this paper for Gaussian distributions and distributions with certain sparsity structures. 
Our error bounds have a favorable dependence on $d$; however, the dependence on the condition number $\beta/\alpha$ may potentially be improved. We also provide counterexamples of a rotated product measure for which the bias is not delocalized. 


\edit{Our work represents a first step in understanding the delocalization of bias phenomenon. Though our sparsity assumption is inspired by physical systems such as statistical physics models and molecular dynamics simulations, it does not fully address those problems. For example, extensions to systems with decaying but long-range interactions and non-logconcave densities should be considered.
In fact, our asymptotic arguments based on Poisson equations provide further insights into how the first-order terms of the bias of general observables scale with dimension.
}

More broadly, it is important to understand how the complexity of other unadjusted MCMC algorithms scales with dimension when the quantities of interest depend solely on low-dimensional marginals. Addressing this question can provide valuable insights to help practitioners select appropriate algorithms and understand their computational complexity when dealing with high-dimensional sampling problems.

\vspace{0.2in}
\noindent {\bf Acknowledgments}
We thank Nawaf Bou-Rabee and Aaron Dinner for helpful discussions. \edit{We thank Sinho Chewi for pointing us to \cite[Theorem 1.2]{altschuler2023shifted}.} Y. Chen is supported in part by the Office of Naval Research project under award N00014-22-1-2728 (PI: Benjamin  Peherstorfer) and Vannevar Bush award ``Mathematical Foundations and Scientific Applications of Machine Learning'' (PI: Eric Vanden-Eijnden). J. Weare and X. Cheng are supported in part by National Science Foundation awards DMS-2054306 and DMS-2425899.
J. Niles-Weed is supported in part by National Science Foundation awards DMS-2210583 and DMS-2339829.

\vspace{0.1in}

\bibliographystyle{plain}
\bibliography{references}
\appendix
\section{Proofs for Log-Concave Product Measures}
\label{appendix-proof-product-measure}
\begin{proposition}
Under Assumption \ref{assumption-V-log-concave}, let $V(x) = \sum_{i=1}^d V_i(x^{(i)})$ such that $\pi \propto \exp(-V)$ is a product measure. Assume $0 \leq h \leq 1/\beta$ and denote $q = \exp(-h\alpha)$. Consider the Langevin dynamics ${\rm d}Y_t = -\nabla V(Y_t){\rm d}t + \sqrt{2}{\rm d}B_t$ with $Y_0 \sim \pi$. Let $\rho_{kh}$ denote the probability distribution of $X_{kh}$ from the iterations of the unadjusted Langevin algorithm with stepsize $h$. Then, 
the following estimates hold for any $k \in \bN$:
    \begin{equation}
    \label{eqn-iterate-contraction}
    W_{2,\ell^{\infty}}(\rho_{(k+1)h},\pi) \leq q W_{2,\ell^{\infty}}(\rho_{kh},\pi) +  \beta\sqrt{(\frac{8\beta}{3}h^4+8h^3)\log(2d)}\, .
    \end{equation}
    This implies that 
    \[W_{2,\ell^{\infty}}(\pi_h,\pi) \leq \frac{4\beta}{\alpha}\sqrt{h\log(2d)}\, .\] 
    
\end{proposition}
\begin{proof}
First, for the continuous-time Langevin dynamics $Y_t$ and the discrete-time iterates $X_{kh}$ in the unadjusted Langevin algorithm, we have:
   \begin{equation}
       \begin{aligned}
           Y_{(k+1)h} &= Y_{kh} - \int_{kh}^{(k+1)h} \nabla V(Y_t) {\rm d} t + \sqrt{2}(B_{(k+1)h} - B_{kh})\\
           X_{(k+1)h} &= X_{kh}-h\nabla V(X_{kh}) + \sqrt{2}(B_{(k+1)h} - B_{kh})\, .
       \end{aligned}
   \end{equation}
   We couple the two processes using the same Brownian motion. We aim to estimate $\bE[|X_{(k+1)h} - Y_{(k+1)h}|_{\infty}^2]$. For this purpose we introduce an auxiliary random variable
   \begin{equation}
       \overline{Y}_{(k+1)h} = Y_{kh} - h \nabla V(Y_{kh}) + \sqrt{2}(B_{(k+1)h} - B_{kh})\, .
   \end{equation}
   Using the triangle inequality then leads to 
   \begin{equation}
       \sqrt{\bE[|X_{(k+1)h} - Y_{(k+1)h}|_{\infty}^2]} \leq \underbrace{\sqrt{\bE[|X_{(k+1)h} - \overline{Y}_{(k+1)h}|_{\infty}^2]}}_{(a)} + \underbrace{\sqrt{\bE[| \overline{Y}_{(k+1)h} - Y_{(k+1)h}|_{\infty}^2]}}_{(b)}\, .
   \end{equation}
   For $(a)$, we have
   \begin{equation}
   \begin{aligned}
       (a) &= \sqrt{\bE[|X_{kh} - Y_{kh} - h(\nabla V(X_{kh}) - \nabla V(Y_{kh}))|_{\infty}^2]}\\
       & = \sqrt{\bE[\max_{1\leq i \leq d}|X_{kh}^{(i)} - Y_{kh}^{(i)} - h(\nabla V_i(X_{kh}^{(i)}) - \nabla V_i(Y_{kh}^{(i)}))|^2]} \\
       & \leq \sqrt{\bE[\max_{1\leq i \leq d} (1-h\alpha)^2|X_{kh}^{(i)} - Y_{kh}^{(i)}|^2]}\\
       & \leq q \sqrt{\bE[|X_{kh} - Y_{kh}|_{\infty}^2]}\, ,
   \end{aligned}
   \end{equation}
   where we used the facts that $V(x) = \sum_{i=1}^d V_i(x^{(i)})$ and $0\leq 1-h\beta \leq 1-h\alpha$. For $(b)$, we have 
   \begin{equation}
   \begin{aligned}
       \bE[| \overline{Y}_{(k+1)h} - Y_{(k+1)h}|_{\infty}^2] & = \bE[|\int_{kh}^{(k+1)h} \nabla V(Y_t) - \nabla V(Y_{kh}) {\rm d}t|_{\infty}^2] \\
       &\leq h \int_{kh}^{(k+1)h} \bE[|\nabla V(Y_t) - \nabla V(Y_{kh})|_{\infty}^2]{\rm d}t \\
       &\leq h\int_{kh}^{(k+1)h} \int_0^1 \bE[|\nabla^2 V(uY_t + (1-u)Y_{kh}) (Y_t -Y_{kh})|_{\infty}^2] {\rm d}u {\rm d} t\\
       & \leq h\beta^2 \int_{kh}^{(k+1)h} \bE[|Y_t - Y_{kh}|_{\infty}^2]{\rm d}t \, ,
   \end{aligned}
   \end{equation}
   where in the last inequality, we used the fact that $\nabla^2 V(uY_t + (1-u)Y_{kh})$ is a diagonal matrix with each diagonal entry bounded by $\beta$ in magnitude. We further have
    \begin{equation}
    \label{eqn-discretization-bound}
    \begin{aligned}
         &\int_{kh}^{(k+1)h} \bE[|Y_t - Y_{kh}|_{\infty}^2]{\rm d}t \\
         =& \int_{kh}^{(k+1)h} \bE[|\int_{kh}^{t} \nabla V(Y_s){\rm d}s + \sqrt{2}B_{t-kh}|_{\infty}^2]{\rm d}t\\
         \leq & \int_{kh}^{(k+1)h} \left(2\bE[|\int_{kh}^{t} \nabla V(Y_s){\rm d}s|_\infty^2] + 2\bE[|\sqrt{2}B_{t-kh}|_{\infty}^2] \right) {\rm d}t\\
         \leq & \int_{kh}^{(k+1)h} 2(t-kh)\int_{kh}^{t} \bE[|\nabla V(Y_s)|_\infty^2]{\rm d}s {\rm d}t + \int_{kh}^{(k+1)h} 2\bE[|\sqrt{2}B_{t-kh}|_{\infty}^2] {\rm d}t\\
        \leq &2\int_{kh}^{(k+1)h} (t-kh)^2\bE_{\pi}[|\nabla V(Y)|_{\infty}^2] {\rm d}t + \int_{kh}^{(k+1)h} 16(t-kh)\log(2d) {\rm d}t\\
         =& \frac{2}{3} h^3 \bE_\pi[|\nabla V(Y)|_{\infty}^2] + 8h^2\log (2d)\, ,
    \end{aligned}
    \end{equation}
    where we used the fact that all $Y_t \sim \pi$, and the bound $\bE|B_u|_{\infty}^2 \leq 4 u\log(2d)$ holds for any $u \geq 0$ due to Lemma \ref{lemma-squared-maximal-inequality-sub-Gaussian}. 
    
    \edit{Then by Proposition \ref{thm-inf-bound}, we have $\sqrt{\bE_\pi[|\nabla V|_{\infty}^2]} \leq 2 \sqrt{\beta\log (2d)}$. Combining all the bounds above, we get
    \[\sqrt{\bE[|X_{(k+1)h} - Y_{(k+1)h}|_{\infty}^2]} \leq q \sqrt{\bE[|X_{kh} - Y_{kh}|_{\infty}^2]} + \beta\sqrt{(\frac{8\beta}{3}h^4+8h^3)\log(2d)}\, .\]
    We can now couple the distribution of $X_{kh}$ and $Y_{kh}$ such that $\sqrt{\bE[|X_{kh} - Y_{kh}|_{\infty}^2]} = W_{2,\ell^{\infty}}(\rho_{kh},\pi)$. With this and using the definition of the $W_{2,\ell^{\infty}}$ norm, we get
    \[W_{2,\ell^{\infty}}(\rho_{(k+1)h},\pi) \leq qW_{2,\ell^{\infty}}(\rho_{kh},\pi) + \beta\sqrt{(\frac{8\beta}{3}h^4+8h^3)\log(2d)}\, . \]
    Iterating this inequality leads to the bound on the $W_{2,\ell^\infty}$ bias. Furthermore under the assumption $h\leq \frac{1}{\beta}$, we get
     \[W_{2,\ell^{\infty}}(\pi_h,\pi) \leq \frac{\beta}{\alpha}\sqrt{(\frac{8\beta h^2}{3}+8h)\log(2d)} \leq \frac{4\beta}{\alpha}\sqrt{h\log(2d)}\, .\] 
     The proof is complete.}
\end{proof}

\section{Proofs for Gaussian Distributions}
\label{appendix-proof-Gaussian-case}
We first state a lemma for the expected squared maximal norm of a random vector whose entries are sub-Gaussians. This lemma will also be used in the proof of our main theorem.
\begin{lemma}
\label{lemma-squared-maximal-inequality-sub-Gaussian}
    Suppose $Y = (Y^{(1)}, Y^{(2)}, ..., Y^{(d)}) \in \bR^d$ and each $Y^{(i)}$ is centered and sub-Gaussian with variance proxy $\sigma^2$, namely
    \begin{equation}
        \bE[\exp(\lambda Y^{(i)})] \leq \exp(\frac{1}{2}\lambda^2\sigma^2)\, .
    \end{equation}
    Then, it holds that
    \begin{equation}
        \bE[|Y|_{\infty}^2] \leq 4\sigma^2 \log(2d)\, .
    \end{equation}
\end{lemma}
\begin{proof}[Proof of Lemma \ref{lemma-squared-maximal-inequality-sub-Gaussian}]
    By the property of sub-Gaussian random variables \cite{wainwright2019high}, we have that for $0\leq \lambda < 1/(2\sigma^2)$,
    \begin{equation}
        \bE[\exp(\lambda (Y^{(i)})^2)] \leq \frac{1}{\sqrt{1-2\lambda \sigma^2}}\, .
    \end{equation}
    Using the convexity of the exponential function, we can derive
    \begin{equation}
     \exp(\lambda \bE[\max_{1\leq i \leq d} (Y^{(i)})^2]) \leq \bE[\exp(\lambda\max_{1\leq i \leq d} (Y^{(i)})^2)] = \bE[\max_{1\leq i \leq d}\exp(\lambda (Y^{(i)})^2)]\, .
    \end{equation}
    Then, we can bound the right hand side as follows:
    \begin{equation}
        \bE[\max_{1\leq i \leq d}\exp(\lambda (Y^{(i)})^2)] \leq \sum_{i=1}^d \bE[\exp(\lambda (Y^{(i)})^2)] \leq \frac{d}{\sqrt{1-2\lambda \sigma^2}}\, .
    \end{equation}
    Thus, combining the above two inequalities and taking logarithms, we find
    \begin{equation}
        \bE[|Y|_{\infty}^2] = \bE[\max_{1\leq i \leq d} (Y^{(i)})^2] \leq \frac{\log d}{\lambda} - \frac{1}{2\lambda}\log(1-2\lambda \sigma^2)\, , 
    \end{equation}
    for any $0\leq \lambda < 1/(2\sigma^2)$. Taking $\lambda = \frac{1}{4\sigma^2}$, we arrive at
    \begin{equation}
        \bE[|Y|_{\infty}^2] \leq 4\sigma^2 \log d + 2\sigma^2 \log 2 \leq 4\sigma^2 \log(2d)\, .
    \end{equation}
    The proof is complete.
\end{proof}

With Lemma \ref{lemma-squared-maximal-inequality-sub-Gaussian}, we present the proof for the statement in Example \ref{intro-gaussian-example} as follows.
\begin{proof}[Proof for Example \ref{intro-gaussian-example}]
When $\pi$ is Gaussian, the potential $V = \frac{1}{2}(x-m)^T\Sigma^{-1}(x-m)$ is quadratic. In this case, the iteration takes the form
 \[ X_{(k+1)h} - m = (I-h\Sigma^{-1})(X_{kh}-m) + \sqrt{2}(B_{(k+1)h} - B_{kh}) \, . \]
Suppose $X_{kh} \sim \mathcal{N}(m_k, \Sigma_k)$, we have
\[m_{k+1} - m = (I-h\Sigma^{-1})(m_k -m ), \Sigma_{k+1} = (I-h\Sigma^{-1})\Sigma_k (I-h\Sigma^{-1}) + 2hI\, .\] Let $0 < h \leq 1/|\Sigma^{-1}|_2 = 1/\beta $, then as $k \to \infty$, we have $m_{\infty} = m$. For $\Sigma_k$, we have the identity:
\begin{equation*}
\begin{aligned}
    \Sigma_{k+1} &= (I - h\Sigma^{-1}) \Sigma_{k} (I - h\Sigma^{-1}) + 2hI\\ 
    &= (I - h \Sigma^{-1})^{k+1}\Sigma_0(I-h\Sigma^{-1})^{k+1} + 2h \sum_{\ell=0}^k(I - h \Sigma^{-1})^{2\ell}\, .
\end{aligned}
\end{equation*}
Letting $k \to \infty$, we get 
\[
\Sigma_{\infty} = 2h (I - (I-h\Sigma^{-1})^2)^{-1} = \Sigma(I - \frac{h}{2} \Sigma^{-1})^{-1}\, .
\]
Thus, $\pi_h = \mathcal{N}(m_\infty, \Sigma_\infty)$ when $\pi = \mathcal{N}(m, \Sigma)$.
Consider the coupling $X = \Sigma^{1/2}Z + m$ and $Y = \Sigma_\infty^{1/2}Z + m$ where $Z \sim \mathcal{N}(0,I)$.
Then, for this specific coupling, we have the bound
\begin{equation}
    W_{2,\ell^{\infty}}^2(\pi, \pi_h) \leq \mathbb{E}[|X-Y|_{\infty}^2] = \mathbb{E} [| (\Sigma^{1/2}-\Sigma^{1/2}_{\infty})Z|_\infty^2]\, .
\end{equation}
Setting $Y' = (\Sigma^{1/2}-\Sigma^{1/2}_{\infty})Z,$ we see that $Y'_i \sim \cN(0, \sigma_i^2)$, where $\sigma_i^2 = ((\Sigma^{1/2}-\Sigma^{1/2}_{\infty})^2)_{ii} \leq |\Sigma^{1/2}-\Sigma^{1/2}_{\infty}|_2^2$.
In particular, the entries of $Y'$ are $|\Sigma^{1/2}-\Sigma^{1/2}_{\infty}|_2^2$-subgaussian.

Consider the eigendecomposition $\Sigma = Q^T\Lambda Q$ where $Q$ is an orthogonal matrix and $\Lambda = \text{diag}(\lambda_1,...,\lambda_d)$; we note that $\alpha \leq \frac{1}{\lambda_i}\leq \beta, 1\leq i \leq d$. Then we have the formula
\[
    \Sigma_\infty =Q^T \Lambda (I - \frac h 2 \Lambda^{-1})^{-1} Q = Q^T \Lambda_\infty Q\, ,
\]
where
\[
    \Lambda_\infty = \mathrm{diag}\left(\frac{\lambda_1}{1 - \frac{h}{2\lambda_1}}, \dots, \frac{\lambda_d}{1 - \frac{h}{2\lambda_d}}\right)\, .
\]

We obtain
\begin{equation}
    |\Sigma^{1/2}-\Sigma^{1/2}_{\infty}|_2 = \max_{1\leq i \leq d}\left|\sqrt{\lambda_i}-\sqrt{\frac{\lambda_i}{1 - \frac{h}{2\lambda_i}}}\right| = O(\sqrt{\beta} h).
\end{equation}

Combining this bound with Lemma~\ref{lemma-squared-maximal-inequality-sub-Gaussian}, we get
$W_{2,\ell^{\infty}}(\pi_h,\pi) = O( \sqrt{\beta}h\sqrt{\log (2d)}) = O(\sqrt{h\log (2d)})$ as $h\leq 1/\beta$.
\end{proof}

\section{Proofs for Log-concave Distributions with Sparse Interactions}
\label{appendix-proof-logconcave distributions with sparse interactions}
\subsection{Sparsity of the propagator of unadjusted Langevin}
\label{appendix-Sparsity-of-the-propagator}
First, we present a proposition concerning the $\ell^{\infty}$ norm of matrices that include the propagator of the Langevin Monte Carlo algorithm as a special case. We will use this proposition when analyzing the convergence of the algorithm in the next subsection.
\begin{proposition}
\label{prop-inf-bound-propagator}
    Assume $0 \leq h \leq 1/\beta$ and denote $q = \exp(-h\alpha)$. For the potential function in Assumption \ref{assume-sparse-potential-function}, the following facts hold:
    \begin{enumerate}
        \item[(i)] The matrix $(\nabla^2 V(x_1))(\nabla^2 V(x_2))...(\nabla^2 V(x_r))$, which is the multiplication of $r$ Hessian matrices, is $s_r$ sparse, for any $x_1,x_2,...,x_r \in \bR^d, r \in \bN$. Here we say a matrix is $s_r$-sparse if each row of the matrix contains at most $s_r$ nonzero entries. 
        \item[(ii)] Let $\nu_1, ..., \nu_N$ be any probability measures in $\bR^d$. Define the matrix 
        \begin{equation}
            P_N = \left(I-h\int \nabla^2 V\, {\rm d}\nu_1\right)\left(I - h\int \nabla^2 V\, {\rm d}\nu_2\right)...\left(I -h\int \nabla^2 V\, {\rm d}\nu_N\right)\, .
        \end{equation}
        We have the inequality
        \begin{equation}
            |P_N|_{\infty}\leq \sqrt{s_r}q^N + \sqrt{d}\exp(-r)
        \end{equation}
        for any $r \geq e^2 Nh\beta$. In particular, taking $r_N = \lceil e^2 Nh\beta + \log \sqrt{d}\rceil$, we get that
        \begin{equation}
            |P_N|_{\infty} \leq 2\sqrt{s_{r_N}}q^N\, .
        \end{equation}
        \item[(iii)] In the context of (ii), consider additionally a probability measure $\nu_0$ and the matrix
        \begin{equation}
            J_N = \left(\int \nabla^2 V\, {\rm d}\nu_0\right)\left(I - h\int \nabla^2 V\, {\rm d}\nu_1\right)\left(I - h\int \nabla^2 V\, {\rm d}\nu_2\right)...\left(I -h\int \nabla^2 V\, {\rm d}\nu_N\right)\, .
        \end{equation}
        We have the inequality
        \begin{equation}
            |J_N|_{\infty}\leq \beta(\sqrt{s_r}q^{N} + \sqrt{d}\exp(-r))
        \end{equation}
        for any $r \geq e^2 Nh\beta$. In particular, taking $r_N = \lceil e^2 Nh\beta + \log \sqrt{d}\rceil$, we get that
        \begin{equation}
            |J_N|_{\infty} \leq 2\beta\sqrt{s_{r_N}}q^N\, .
        \end{equation}
    \end{enumerate}
\end{proposition}
\begin{proof}[Proof of Proposition \ref{prop-inf-bound-propagator}] We prove the above facts one by one.\\
     \textbf{Proof for (i).} We know that the $ij$-th entry of the matrix \[(\nabla^2 V(x_1))(\nabla^2 V(x_2))...(\nabla^2 V(x_r))\] is nonzero if $j \in \sfN_r(i)$, or if there exists an $k$ such that $i \sim \sfN_r(k), j \sim \sfN_r(k)$. \edit{Since $\max_{1\leq i \leq d} |\sfN_{2r}(i)| = s_r$, we will have at most $s_r$ nonzero entries in each row. This implies that the matrix is $s_r$-sparse.}
     \\
     \textbf{Proof for (ii).} Let us denote $A_k = \int \nabla^2 V\, {\rm d}\nu_k$, then $P_N = (I-hA_1)(I-hA_2)\cdots(I-hA_N)$. Expanding the product, we have
     \begin{equation}
     \begin{aligned}
         P_N &= \sum_{k=0}^N(-1)^k h^k \sum_{1\leq i_1 < ... < i_k \leq N}A_{i_1}...A_{i_k} \\
         & = \sum_{k=0}^r (-1)^k h^k \sum_{1\leq i_1 < ... < i_k \leq N}A_{i_1}...A_{i_k} + \sum_{k=r+1}^N (-1)^k h^k \sum_{1\leq i_1 < ... < i_k \leq N}A_{i_1}...A_{i_k}\, ,
     \end{aligned}
     \end{equation}
     For $k\geq r+1$, we bound
     \begin{equation}
         \left|(-1)^k h^k \sum_{1\leq i_1 < ... < i_k \leq N}A_{i_1}...A_{i_k}\right|_2 \leq h^k\beta^k {N\choose k} \leq \frac{(Nh\beta)^k}{k!}\leq (\frac{eNh\beta}{k})^k\, ,
     \end{equation}
     where in the last inequality, we used the fact that $k! \geq k^k /\exp(k)$. This fact can be seen by taking $x = k$ in the inequality $\exp(x)\geq x^k/k!$. 

     When $r \geq e^2Nh\beta$, we can bound
     \begin{equation}
         \left|\sum_{k=r+1}^N (-1)^k h^k \sum_{1\leq i_1 < ... < i_k \leq N}A_{i_1}...A_{i_k}\right|_2 \leq \sum_{k=r+1}^N (\frac{eNh\beta}{k})^k \leq \sum_{k=r+1}^N \exp(-k) \leq \exp(-r)\, .
     \end{equation}
     Furthermore, we obtain $\left|\sum_{k=r+1}^N (-1)^k h^k \sum_{1\leq i_1 < ... < i_k \leq N}A_{i_1}...A_{i_k}\right|_{\infty} \leq \sqrt{d}\exp(-r)$ by the Cauchy-Schwarz inequality. 
     
     On the other hand, $|P_N|_2 \leq \Pi_{k=1}^N |I-hA_k|_2 \leq (1-h\alpha)^N \leq q^N$ where we used the fact that $h \leq 1/\beta$. Then, for each row $i$ of $P_N$,
     \begin{equation}
     \begin{aligned}
         \sum_{j=1}^N |(P_N)_{ij}| &= \sum_{j \in \sfN_r(i)} |(P_N)_{ij}| + \sum_{j \notin \sfN_r(i)} |(P_N)_{ij}|\\
         & \leq \sqrt{|\sfN_r(i)|} \sqrt{\sum_{j \in \sfN_r(i)} |(P_N)_{ij}|^2} + \sqrt{d}\exp(-r)\\
         & \leq \sqrt{s_r}\sqrt{\sum_{j=1}^d |(P_N)_{ij}|^2} + \sqrt{d}\exp(-r)\\
         & \leq \sqrt{s_r}q^N + \sqrt{d}\exp(-r)\, ,
     \end{aligned}
     \end{equation}
     where in the second inequality, we used the fact that 
     \[\sum_{j \notin \sfN_r(i)} |(P_N)_{ij}| \leq |\sum_{k=r+1}^N (-1)^k h^k \sum_{1\leq i_1 < ... < i_k \leq N}A_{i_1}...A_{i_k}|_{\infty}\, , \]
     since a nonzero term $(P_N)_{ij}$ for $j \notin \sfN_r(i)$ can only be produced by product of more than $r$ matrices of the kind of $A_{i_k}$, as a consequence of the argument in (i). In summary, we get the result that if $r \geq e^2 Nh\beta$, then 
     \[|P_N|_{\infty} = \max_i \sum_{j=1}^N |(P_N)_{ij}| \leq \sqrt{s_r}q^N + \sqrt{d}\exp(-r)\, .\]
    We note that the above inequality will also hold true when $r \geq d$.
     
     Taking $r_N = \lceil e^2 Nh\beta + \log \sqrt{d}\rceil$, we get
     \begin{equation}
         \begin{aligned}
             |P_N|_{\infty}&\leq \sqrt{s_{r_N}}q^N + \sqrt{d}\exp(-r_N) \\
             & \leq \sqrt{s_{r_N}}\exp(-h\alpha N ) + \sqrt{d}\exp(-e^2 Nh\beta - \log \sqrt{d})\\
             & \leq \sqrt{s_{r_N}}\exp(-h\alpha N ) + \exp(-e^2 Nh\beta) \leq 2\sqrt{s_{r_N}}\exp(-h\alpha N ) = 2\sqrt{s_{r_N}}q^N\, .
         \end{aligned}
     \end{equation}
     \textbf{Proof for (iii).} Note that $J_N = \left(\int \nabla^2 V\, {\rm d}\nu_0\right) P_N$. Following the proof for (ii), we have that when $r \geq e^2Nh\beta$, 
     \begin{equation}
     \begin{aligned}
         &\left|\left(\int \nabla^2 V\, {\rm d}\nu_0\right) \sum_{k=r+2}^N (-1)^k h^k \sum_{1\leq i_1 < ... < i_k \leq N}A_{i_1}...A_{i_k}\right|_2 \\
         \leq &\beta\sum_{k=r+2}^N (\frac{eNh\beta}{k})^k \leq \beta\sum_{k=r+2}^N \exp(-k) \leq \beta\exp(-r)\, .
     \end{aligned}
     \end{equation}
     Thus the $\infty$-norm of the above matrix is bounded by $\sqrt{d}\beta\exp(-r)$. Then, for any $1\leq i \leq d$, 
     \begin{equation}
     \begin{aligned}
         \sum_{j=1}^d |(J_N)_{ij}| \leq \beta \sqrt{s_r}q^{N} + \sqrt{d}\beta\exp(-r)\, .
         \end{aligned}
    \end{equation}
    Therefore, $ |J_N|_{\infty}\leq \beta(\sqrt{s_r}q^{N} + \sqrt{d}\exp(-r))$ for $r \geq e^2Nh\beta$. Taking $r_N = \lceil e^2 Nh\beta + \log \sqrt{d}\rceil$, we get $|J_N|_{\infty} \leq 2\beta\sqrt{s_{r_N}}q^N$. The proof is complete.
\end{proof}

\subsection{A multistep coupling argument}
To prove Theorem \ref{thm-OLD-bias-under-sparsity}, we first show the following Proposition \ref{prop-bias-under-sparsity}, which bounds the $W_{2,\ell^{\infty}}$ metric through a multistep coupling argument. Note that the one-step coupling argument that is commonly used in bounding $W_2$ distance is not enough here, as the one-step contraction property is lost in the $W_{2,\ell^{\infty}}$ metric.
\begin{proposition}
\label{prop-bias-under-sparsity}
Let Assumptions \ref{assumption-V-log-concave} and \ref{assume-sparse-potential-function} hold. Assume $0 \leq h \leq 1/\beta$ and denote $q = \exp(-h\alpha)$. Consider the Langevin dynamics ${\rm d}Y_t = -\nabla V(Y_t){\rm d}t + \sqrt{2}{\rm d}B_t$ with $Y_0 \sim \pi$. Let $\rho_{kh}$ represent the law of $X_{kh}$ from the iterations of Langevin Monte Carlo with stepsize $h$. Then, 
the following estimates hold for any $k, N \in \bN$:
    \begin{equation}
    \label{eqn-iterate-contraction}
    W_{2,\ell^{\infty}}(\rho_{(k+N)h},\pi) \leq 2q^N\sqrt{s_{r_N}}W_{2,\ell^{\infty}}(\rho_{kh},\pi) + 2\beta(\sum_{i=1}^{N} q^{i-1}\sqrt{hs_{r_i}}\epsilon_{k+N-i})\, , 
    \end{equation}
    where $r_i, 1\leq i \leq N$ is defined in Proposition \ref{prop-inf-bound-propagator}, and $\epsilon_{j}^2 = \int_{jh}^{(j+1)h} \bE[|Y_t - Y_{jh}|^2_{\infty}]{\rm d}t$ for $j \in \bN$, which satisfies
    \[\epsilon_j \leq h^{3/2}\sqrt{\bE_\pi [|\nabla V(Y)|_{\infty}^2]} + 3h\sqrt{\log(2d)}\, .\] 
\end{proposition}
\begin{proof}[Proof of Proposition \ref{prop-bias-under-sparsity}]
Let ${\rm d}Y_t = -\nabla V(Y_t){\rm d}t + \sqrt{2}{\rm d}B_t$. We write down the following identity:
   \begin{equation}
       \begin{aligned}
           Y_{(k+1)h} &= Y_{kh} - \int_{kh}^{(k+1)h} \nabla V(Y_t) {\rm d} t + \sqrt{2}(B_{(k+1)h} - B_{kh})\\
           X_{(k+1)h} &= X_{kh}-h\nabla V(X_{kh}) + \sqrt{2}(B_{(k+1)h} - B_{kh})\, .
       \end{aligned}
   \end{equation}
   We couple the two processes using the same Brownian motion. The goal is to estimate $\bE[|X_{(k+1)h} - Y_{(k+1)h}|_{\infty}^2]$. To do so we introduce an auxiliary random variable
   \begin{equation}
       \overline{Y}_{(k+1)h} = Y_{kh} - h \nabla V(Y_{kh}) + \sqrt{2}(B_{(k+1)h} - B_{kh})\, .
   \end{equation}
   Using the triangle inequality then leads to 
   \begin{equation}\label{eq: onestepdiff}
       \sqrt{\bE[|X_{(k+1)h} - Y_{(k+1)h}|_{\infty}^2]} \leq \underbrace{\sqrt{\bE[|X_{(k+1)h} - \overline{Y}_{(k+1)h}|_{\infty}^2]}}_{(a)} + \underbrace{\sqrt{\bE[| \overline{Y}_{(k+1)h} - Y_{(k+1)h}|_{\infty}^2]}}_{(b)}\, .
   \end{equation}
   For $(a)$, we have
   \begin{equation}
   \begin{aligned}
       (a) &= \sqrt{\bE[|X_{kh} - Y_{kh} - h(\nabla V(X_{kh}) - \nabla V(Y_{kh}))|_{\infty}^2]}\\
       & = \sqrt{\bE[|X_{kh} - Y_{kh} - h (\int_0^1 \nabla^2 V (uX_{kh}+(1-u)Y_{kh}) {\rm d}u)(X_{kh} - Y_{kh})|_{\infty}^2]}\\
       & = \sqrt{\bE[|H_k(X_{kh}-Y_{kh})|_{\infty}^2]}\, ,
   \end{aligned}
   \end{equation}
   where $H_k(X_{kh}-Y_{kh}) = (I - h \int_0^1 \nabla^2 V (uX_{kh}+(1-u)Y_{kh}){\rm d}u)(X_{kh}-Y_{kh})$. We can view $H_k$ as a random matrix depending on $X_{kh}$ and $Y_{kh}$. In particular, each realization of this random matrix can be written in the form $\left(I-h\int \nabla^2 V\, {\rm d}\nu\right)$ where $\nu$ is a probability measure; here $\nu$ has its mass on a line. This form allows us to use Proposition \ref{prop-inf-bound-propagator} to analyze the iterations.

   For $(b)$, we have 
   \begin{equation}
   \begin{aligned}
       \bE[| \overline{Y}_{(k+1)h} - Y_{(k+1)h}|_{\infty}^2] & = \bE[|\int_{kh}^{(k+1)h} \nabla V(Y_t) - \nabla V(Y_{kh}) {\rm d}t|_{\infty}^2] \\
       &\leq h \int_{kh}^{(k+1)h} \bE[|\nabla V(Y_t) - \nabla V(Y_{kh})|_{\infty}^2]{\rm d}t \\
       &\leq h\int_{kh}^{(k+1)h} \int_0^1 \bE[|\nabla^2 V(uY_t + (1-u)Y_{kh}) (Y_t -Y_{kh})|_{\infty}^2] {\rm d}u {\rm d} t\\
       & \leq hs_1\beta^2 \int_{kh}^{(k+1)h} \bE[|Y_t - Y_{kh}|_{\infty}^2]{\rm d}t = hs_1 \beta^2 \epsilon_k^2 \, ,
   \end{aligned}
   \end{equation}
   which can be understood as the one-step discretization error of the overdamped Langevin dynamics. In the above, we used the fact that $\nabla^2 V(uY_t + (1-u)Y_{kh})$ is $s_1$-sparse and $|\nabla^2 V(uY_t + (1-u)Y_{kh})|_2 \leq \beta$; thus $|\nabla^2 V(uY_t + (1-u)Y_{kh})|_{\infty} \leq \sqrt{s_1} \beta$. We overestimate the right hand side $s_1 \beta^2 \leq 4 s_{r_1} \beta^2$.
   
    Now, combining the above estimates, we arrive at
   \begin{equation}
       \sqrt{\bE[|X_{(k+1)h} - Y_{(k+1)h}|_{\infty}^2|} \leq \sqrt{\bE[|H_k(X_{kh}-Y_{kh})|_{\infty}^2]} + 2\sqrt{h s_{r_1}}\beta \epsilon_k\, .
   \end{equation}
   We move one step back further and get
   \begin{equation}
       \sqrt{\bE[|H_k(X_{kh}-Y_{kh})|_{\infty}^2]} \leq \underbrace{\sqrt{\bE[|H_k(X_{kh}-\overline{Y}_{kh})|_{\infty}^2]}}_{(c)} + \underbrace{\sqrt{\bE[|H_k(\overline{Y}_{kh}-Y_{kh})|_{\infty}^2]}}_{(d)}\, ,
   \end{equation}
   where, similar as before, we define $\overline{Y}_{kh} = Y_{(k-1)h} - h\nabla V(Y_{(k-1)h}) + \sqrt{2}(B_{kh} - B_{(k-1)h})$. For $(c)$, we use the same argument as earlier to get \[(c) = \sqrt{\bE[|H_kH_{k-1}(X_{(k-1)h}-{Y}_{(k-1)h})|_{\infty}^2]}\, .\]
   For $(d)$, it holds that
   \begin{equation*}
   \begin{aligned}
       &\bE[|H_k (\overline{Y}_{kh} - Y_{kh})|_{\infty}^2] \\
       \leq &h\int_{(k-1)h}^{kh} \bE[|H_k(\nabla V(Y_t) - \nabla V(Y_{(k-1)h}))|_{\infty}^2]{\rm d}t \\
       \leq &h\int_{(k-1)h}^{kh} \int_0^1 \bE[|H_k(\nabla^2 V(uY_t + (1-u)Y_{(k-1)h})) (Y_t -Y_{(k-1)h})|_{\infty}^2] {\rm d}u {\rm d} t\\
       \leq & 4h s_{r_2} \beta^2 q^2 \int_{(k-1)h}^{kh} \bE[|Y_t - Y_{(k-1)h}|_{\infty}^2]{\rm d}t = 4h s_{r_2} \beta^2 q^2 \epsilon_{k-1}^2 \, ,
   \end{aligned}
   \end{equation*}
   where we applied (iii) of Proposition \ref{prop-inf-bound-propagator} to the the matrix $H_k\nabla^2 V(uY_t + (1-u)Y_{(k-1)h})$. \edit{Indeed, the bound will depend on $s_{r_1}$ if we use the argument in Proposition \ref{prop-inf-bound-propagator} directly. Here, we overestimate the bound by bounding $s_{r_1}$ by $s_{r_2}$, which helps organize the terms in a unified way.}
   
   As a summary, we get
   \begin{equation}
   \begin{aligned}
       &\sqrt{\bE[|X_{(k+1)h} - Y_{(k+1)h}|_{\infty}^2]} \\
       \leq &\sqrt{\bE[|H_kH_{k-1}(X_{(k-1)h}-{Y}_{(k-1)h})|_{\infty}]} + 2\beta (\sqrt{hs_{r_1}}\epsilon_k + \sqrt{s_{r_2}h}q \epsilon_{k-1})\, .
   \end{aligned}
   \end{equation}
   Iterating the above arguments $N$ times, and for simplicity of notations writing $k+N$ in place of $k+1$, we get
   \begin{equation}
       \begin{aligned}
           &\sqrt{\bE[|X_{(k+N)h} - Y_{(k+N)h}|_{\infty}^2]} \\
           \leq & \sqrt{\bE[|H_{k+N-1}H_{k+N-2}\cdots H_{k}(X_{kh}-{Y}_{kh})|_{\infty}^2]} +  2\beta(\sum_{i=1}^N q^{i-1}\sqrt{hs_{r_i}}\epsilon_{k+N-i})\, ,\\
           \leq & 2q^N \sqrt{s_{r_N}} \sqrt{\bE[|X_{kh} - Y_{kh}|_{\infty}^2]} + 2\beta(\sum_{i=1}^N q^{i-1}\sqrt{hs_{r_i}}\epsilon_{k+N-i})\, ,
       \end{aligned}
   \end{equation}
   where in the last step, we applied (ii) of Proposition \ref{prop-inf-bound-propagator} to the matrix $H_{k+r-1}H_{k+r-2}\cdots H_{k}$. We can now couple the distribution of $X_{kh}$ and $Y_{kh}$ such that $\sqrt{\bE[|X_{kh} - Y_{kh}|_{\infty}^2]} = W_{2,\ell^{\infty}}(\rho_{kh},\pi)$. With this and using the definition of the $W_{2,\ell^{\infty}}$ norm, we get
    \[W_{2,\ell^{\infty}}(\rho_{(k+N)h},\pi) \leq 2q^N\sqrt{s_{r_N}}W_{2,\ell^{\infty}}(\rho_{kh},\pi) + 2\beta(\sum_{i=1}^N q^{i-1}\sqrt{hs_{r_i}}\epsilon_{k+N-i})\, . \]
    With this expression, we can ensure contraction in the $W_{2, \ell^{\infty}}$ metric by iterating sufficiently large number of steps $N$, so that the factor $q^N$ is small enough to offset the increasing sparsity parameter $s_{r_N}$. In particular, in the dense case where $s_{r_N} = d$, such a contraction can be achieved by taking $N \sim \frac{\log d}{h}$.
    
    We can further bound $\epsilon_j^2$ using the same approach as in \eqref{eqn-discretization-bound}, which implies that
    \begin{equation}
        \epsilon_j \leq h^{3/2}\sqrt{\bE_\pi [|\nabla V(Y)|_{\infty}^2]} + 3h\sqrt{\log(2d)}\, .
    \end{equation}
    The proof is complete.
\end{proof}
\subsection{Convergence bounds}
\label{sec-proof-convergence-polynomial-sparsity}
\begin{proof}[Proof of Theorem \ref{thm-OLD-bias-under-sparsity}]
    With the bound on $\epsilon_j$ in Proposition \ref{prop-bias-under-sparsity}, we take the limit $k \to \infty$ in \eqref{eqn-iterate-contraction} to obtain the bias bound. 

   Once we have the bound on the bias, we can utilize the convergence bound for $W_2$ to establish the following:
    \begin{equation}
    \begin{aligned}
        W_{2,\ell^{\infty}}(\rho_{kh}, \pi) 
\leq W_{2,\ell^{\infty}}(\rho_{kh}, \pi_h) + W_{2,\ell^{\infty}}(\pi_h, \pi) &\leq W_2(\rho_{kh}, \pi_h) + W_{2,\ell^{\infty}}(\pi_h, \pi) \\
&\leq q^k W_2(\rho_0, \pi_h) + \mathsf{bias}(N, q, \beta, h)\, ,
    \end{aligned}
\end{equation}
where we used the contraction in $W_2$ to get $W_2(\rho_{kh}, \pi_h)\leq q^k W_2(\rho_0, \pi_h)$. The proof is complete.
\end{proof}

\begin{proof}[Proof of Proposition \ref{thm-inf-bound}]
Under the assumption, it holds that $\nabla V$ is $\sqrt{\beta}$ sub-Gaussian under $\pi$ \cite[Theorem 1.2]{altschuler2023shifted}. 
Thus, by Lemma \ref{lemma-squared-maximal-inequality-sub-Gaussian}, we get
\begin{equation}
    \sqrt{\bE_{\pi} [|\nabla V(Y) - \bE_{\pi}[\nabla V(Y)] |_{\infty}^2]} \leq 2\sqrt{\beta \log (2d)}\, .
\end{equation}
Note that $\bE_{\pi} [\nabla V(Y)] = -\int \pi \nabla \log \pi  = \int \nabla \pi = 0$. The proof is complete.
\end{proof} 

\begin{proof}[Proof of Theorem \ref{prop-polynomial-growth-s}]
    We will use Theorem \ref{thm-OLD-bias-under-sparsity}. Recall the definition $r_N = \lceil e^2 Nh\beta + \log \sqrt{d}\rceil$ and $q = \exp(-h\alpha)$. There is a universal upper bound on $s_i$, given by $s_i \leq d$. Choose $N = \lceil \frac{\log (4\sqrt{d})}{h\alpha} \rceil$, which leads to $2q^N \sqrt{s_{r_N}} \leq 1/2$. Based on Theorem \ref{thm-OLD-bias-under-sparsity}, it remains to calculate the bound on $\sum_{i=1}^N q^{i-1}\sqrt{s_{r_i}}$.

    Since $q \leq 1$, we have the bound
    \begin{equation}
    \begin{aligned}
        \sum_{i=1}^N q^{i-1}\sqrt{s_{r_i}} &\leq \sum_{i=1}^N \sqrt{C}(e^2 ih\beta + \log \sqrt{d}+2)^{\frac{n}{2}}\\
        & \leq \sqrt{C} \int_1^{N+1} (e^2 yh\beta + \log \sqrt{d}+2)^{\frac{n}{2}}\, {\rm d}y\\
        & \leq \sqrt{C}\frac{\left(e^2 (N+1)h\beta + \log \sqrt{d}+2\right)^{\frac{n}{2}+1}}{(n/2+1)(e^2h\beta)}\\
        &\leq \sqrt{C}\frac{\left(4e^2\log (4\sqrt{d}) \frac{\beta}{\alpha} + \log \sqrt{d}+2\right)^{\frac{n}{2}+1}}{(n/2+1)(e^2h\beta)} \, .
    \end{aligned}
    \end{equation}
    Therefore, we get
    \begin{equation}
    \begin{aligned}
        \frac{\beta}{1-2q^N\sqrt{s_{r_N}}}(2\sum_{i=1}^N q^{i-1}\sqrt{s_{r_i}}) &\leq 4\sqrt{C}\frac{\left(4e^2\log (4\sqrt{d}) \frac{\beta}{\alpha} + \log \sqrt{d}+2\right)^{\frac{n}{2}+1}}{(n/2+1)(e^2h)}
        \\
        =& \frac{1}{h}\left(O\big(\frac{\beta}{\alpha}\log(2d)\big)\right)^{\frac{n}{2}+1} \, .
    \end{aligned}
    \end{equation}
    On the other hand, the trivial bound $s_i \leq d$ leads to
    \begin{equation}
        \sum_{i=1}^N q^{i-1}\sqrt{s_{r_i}} \leq \frac{\sqrt{d}}{\alpha h}\, .
    \end{equation}
    Thus, we have another bound
    \begin{equation}
        \frac{\beta}{1-2q^N\sqrt{s_{r_N}}}(2\sum_{i=1}^N q^{i-1}\sqrt{s_{r_i}}) \leq \frac{4\beta \sqrt{d}}{\alpha h}\, .
    \end{equation}
    Moreover, by Proposition \ref{thm-inf-bound}, we have
    \begin{equation}
        h^2\sqrt{\bE_\pi [|\nabla V(Y)|_{\infty}^2]} + 3h^{3/2}\sqrt{\log(2d)} \lesssim \left(\sqrt{\beta}h^2 + h^{3/2}\right)\sqrt{\log(2d)}\, .
    \end{equation}
    Using the fact that $h \leq \frac{1}{\beta}$ and combining the above two inequalities leads to the final result.
\end{proof}
\section{Asymptotic Bias for General Observables}
\label{appendix-A Formula for the Asymptotic Bias}
In the proof, we use the notation $x = (x_1,...,x_d) \in \bR^d$. \edit{The assumptions used in Proposition \ref{prop-formula} are
\begin{itemize}
    \item Both $u, f$ are smooth, so that the pointwise Taylor expansion is valid.
    \item $\lim_{h\to 0} \frac{1}{h}\int  \pi(\cL_h u-\cL u) =  \int  \pi (\lim_{h\to 0}\frac{\cL_h u-\cL u}{h}).$
    \item $\lim_{h\to 0}\frac{1}{h}\int (\pi-\pi_h)(\cL_h u-\cL u)=0.$
\end{itemize}
These assumptions are made for technical reasons and arise naturally in the Taylor expansion calculations. They may be verified on a case-by-case basis. Here, for simplicity, we assume they hold and use them to derive the asymptotic expansion of the bias of certain observables.}
\begin{proof}[Proof of Proposition \ref{prop-formula}]
Let $\cL u = f$. Then, we get $\int f\pi - \int f\pi_h = -\int \cL u \pi_h = \int (\cL_h u-\cL u) \pi_h$. 
\begin{equation}
    \cL_h u(x) = \frac{\bE[u(\bar{x}+ \sqrt{2h}\xi)] - u(\bar{x}) + u(\bar{x}) - u(x)}{h}\, ,
\end{equation}
where $\bar{x} = x+h \nabla \log \pi(x)$. Then, by Taylor's expansion, we get 
\begin{equation}
\begin{aligned}
     &\bE[u(\bar{x}+ \sqrt{2h}\xi)] - u(\bar{x})\\
     =& \frac{1}{2}\cdot 2h\cdot \bE[\xi^T\nabla^2 u(\bar{x})\xi] + \sum_{|{\bm \alpha}|=4}\frac{1}{{\bm \alpha}!}D^{{\bm \alpha}}u(\bar{x})\bE[\xi^{{\bm \alpha}}]\cdot (4h^2) + o(h^2)\\
     = & h\Delta u(\bar{x}) + h^2\left(\frac{1}{2}\sum_{i=1}^d D_{i}^4 u(\bar{x}) + \sum_{1\leq i<j \leq d} D_{i}^2D_{j}^2 u(\bar{x})\right) + o(h^2)\\
     = & h\Delta u(x) + h^2(\nabla \Delta u(x))\cdot \nabla \log \pi(x) + h^2\left(\frac{1}{2}\sum_{i=1}^d D_{i}^4 u(x) + \sum_{1\leq i<j \leq d} D_{i}^2D_{j}^2 u(x)\right) + o(h^2)\\
     = &h\Delta u(x) + h^2(\nabla \Delta u(x))\cdot \nabla \log \pi(x) + \frac{1}{2}h^2\Delta^2 u + o(h^2)\, ,
\end{aligned}
\end{equation}
where in the first identity, we used the notation that ${\bm \alpha} = (\alpha_1,...,\alpha_4)$ is a multi-index with non-negative entries and $|{\bm \alpha}|=4$ implies that $\sum_i \alpha_i = 4$. Moreover ${\bm \alpha}! = \alpha_1! \cdot ...\cdot \alpha_d!$ and $\xi^{{\bm \alpha}} = \xi_1^{\alpha_1}\cdot ... \cdot \xi_d^{\alpha_d}$. In the second identity, we used the notation $D_i^k u = \frac{\partial^k}{\partial x_i^k}u$ and the fact that $\bE[\xi_i^2]=1$ and $\bE[\xi_i^4]=3$ for $\xi_i \sim \cN(0,1)$. In the third identity, we performed the Taylor expansion at $x$, based on the fact $\bar{x} = x+h \nabla \log \pi(x)$. In the last identity, we noticed the fact that the terms in the big bracket equal $\frac{1}{2}\Delta^2 u$.

Moreover, $u(\bar{x}) - u(x) = h\nabla u(x) \cdot \nabla \log \pi(x) + \frac{1}{2}h^2(\nabla \log \pi(x))^T\nabla^2 u(x) \nabla \log \pi(x) + o(h^2)$. 
Therefore, we get
\begin{equation}
\label{eqn-taylor-Lu-point-wise}
\begin{aligned}
    \cL_h u(x)-\cL u(x) = & h(\nabla \Delta u(x))\cdot \nabla \log \pi(x) + \frac{1}{2}h(\nabla \log \pi(x))^T\nabla^2 u(x) \nabla \log \pi(x) \\ &+  \frac{1}{2} h\Delta^2 u(x) + o(h)\, ,
\end{aligned}
\end{equation}
where we have used the definition $\cL u = \nabla \log \pi \cdot \nabla u + \Delta u$.

We note that
\begin{equation}
    \int \pi \nabla \Delta u\cdot \nabla \log \pi = \int \nabla \Delta u\cdot \nabla \pi = -\int \pi \Delta^2 u\, .
\end{equation}
Therefore, we get
\begin{equation}
\label{eqn-taylor-averaged}
    \int \pi(\cL_h u-\cL u) = \frac{1}{2}h\int \pi \left(\nabla \Delta u\cdot \nabla \log \pi + (\nabla \log \pi)^T(\nabla^2 u) \nabla \log \pi)\right) + o(h)\, .
\end{equation}
\edit{Note that in the above, going from the pointwise result \eqref{eqn-taylor-Lu-point-wise} to the averaged result \eqref{eqn-taylor-averaged} rigorously requires some justification. Here, it is guaranteed by assuming
\[\lim_{h\to 0} \frac{1}{h}\int  \pi(\mathcal{L}_h u-\mathcal{L} u) =  \int  \pi \lim_{h\to 0}\frac{\mathcal{L}_h u-\mathcal{L} u}{h}\, . \]
}

Now, note that, 
\begin{equation}
\begin{aligned}
    \nabla f &= \nabla \cL u = \nabla (\nabla \log \pi\cdot \nabla u + \Delta u)\\\
    & = (\nabla^2\log\pi) \nabla u + (\nabla^2 u) \nabla \log \pi + \nabla \Delta u\, ,
\end{aligned}
\end{equation}
we get 
\begin{equation}
    \int \pi(\cL_h u-\cL u) = \frac{1}{2}h\int \pi \Big( \nabla f \cdot \nabla \log \pi -(\nabla\log\pi)^T (\nabla^2\log \pi) \nabla u\Big) + o(h)\, .
\end{equation}
 Let $g = \frac{1}{2}|\nabla \log \pi|^2_2$ which satisfies the equation $\nabla g = (\nabla^2 \log \pi) \nabla\log\pi$. We know that the adjoint of the generator satisfies $\cL^*(g\pi) = \nabla \cdot (\pi \nabla g)$.
Thus
\begin{equation}
\begin{aligned}
    &\int \pi(\nabla\log\pi)^T (\nabla^2\log \pi) \nabla u \\
    =& -\int \nabla \cdot (\pi \nabla g)u \\
    = &-\int (\cL^*(g\pi))u \\
    = &-\int gf\pi\, .
\end{aligned}
\end{equation}
Therefore, 
\begin{equation}
\begin{aligned}
    \int f\pi - \int f\pi_h &=  \int \pi_h(\cL_h u-\cL u)  = \int \pi(\cL_h u-\cL u) + o(h) \\
    & = \frac{1}{2}h\int \pi \Big( \nabla f \cdot \nabla \log \pi +\frac{1}{2}|\nabla \log \pi|^2_2f\Big) + o(h)\\
    & = \frac{1}{4}h \int (-2\Delta f + |\nabla \log \pi|^2_2f)\pi + o(h)\, ,
\end{aligned}
\end{equation}
\edit{where in the first line, we have used the assumption that $\int (\pi-\pi_h)(\cL_h u-\cL u)=o(h)$}. In the last identity, we used the fact $\pi \nabla \log \pi = \nabla \pi$ and integration by parts.

 Moreover, note that $\nabla \pi = \pi \nabla \log \pi$ and thus \[\Delta \pi = \nabla \pi \cdot \nabla \log \pi  + \pi\Delta \log \pi = \pi|\nabla \log \pi|_2^2 + \pi\Delta \log \pi\, , \]
    which implies that
    \[\int |\nabla \log \pi|^2_2f\pi =\int f\Delta \pi - f\pi\Delta \log \pi = \int \pi\Delta f - f\pi\Delta \log \pi\, .  \]
    Therefore, we also have another representation of the bias
    \begin{equation}
        \int f\pi - \int f\pi_h = -\frac{1}{4}h \int (\Delta f + f\Delta \log \pi)\pi + o(h)\, .
    \end{equation}
\end{proof}
\begin{proof}[Proof of Proposition \ref{prop-gaussian-perturbation}]
We have
\edit{
\begin{equation}
    \begin{aligned}
         \left|\int f (\Delta \log \pi)\pi \right| &= \left|\int f (\Delta U )\pi \right|\\
         & =\left|\int f (\Delta U - (\int \Delta U \pi) )\pi \right|\\
         &\leq \sqrt{\int (f)^2\pi}\sqrt{\int (\Delta U - (\int \Delta U \pi))^2\pi} \\
         & \leq \sqrt{C_0} \sqrt{\int (f)^2\pi}\, .
    \end{aligned}
\end{equation}
}
When $f$ only depends on $K$ number of coordinates of $x \in \bR^d$, then the integral $\int (f)^2\pi$ will only scale with that number, rather than $d$, because the whole integral will only rely on the marginal distribution of $\pi$ at these coordinates. As such argument also applies to the term $-\frac{1}{4}\int (\Delta f) \pi$, we get that the first order term of the bias depend only on $K$ rather than $d$.

In particular, for $f(x) = \sum_{k=1}^K x^{(k)}$, we have $\Delta f = 0$. Then, using the assumed Poincar\'e inequality, 
we get $\sqrt{\int (f)^2\pi} \leq \sqrt{C_{\text{PI}}\int |\nabla f|^2\pi} = O(\sqrt{K})$ where $C_{\text{PI}}$ is the Poincar\'e constant. Thus, in such case,
\begin{equation}
    \left|\int f (\Delta \log \pi)\pi \right| = O(\sqrt{K}h) + o(h)\, .
\end{equation}
The proof is complete.
\end{proof}

\end{document}